\documentclass{article}

\usepackage[nonatbib, final]{neurips_2021}

\usepackage[ruled]{algorithm2e}
\usepackage{graphicx}
\usepackage{caption}
\usepackage[utf8]{inputenc} 
\usepackage[T1]{fontenc}    
\usepackage{hyperref}       
\usepackage{url}            
\usepackage{booktabs}       
\usepackage{amsfonts}       
\usepackage{nicefrac}       
\usepackage{microtype}      
\usepackage{xcolor}         
\usepackage{enumitem}
\usepackage{amsfonts}
\usepackage{amsmath}
\usepackage{amssymb}
\usepackage{amstext}
\usepackage{amsthm}
\usepackage{tikz}
\usepackage{pgfplots}

\usepackage{url, hyperref}
\usepackage{amsthm, amsmath, amssymb} 
\usepackage{color}
\usepackage{multirow}
\usepackage{xspace}
\usepackage{rotating}
\usepackage{graphicx}
\usepackage{caption}
\usepackage{xcolor}
\usepackage[normalem]{ulem}
\usepackage{tikz}

\usepackage{nicefrac}
\usepackage{mathtools}

\theoremstyle{plain}
\newtheorem{theorem}{\textbf{Theorem}}
\newtheorem{lemma}[theorem]{\textbf{Lemma}}

\newtheorem{proposition}[theorem]{Proposition}

\theoremstyle{definition}

\newtheorem{assumption}{Assumption}

\renewcommand{\cite}{\citep}

\newcommand{\para}[1]{\noindent \textbf{#1}~}

\newcommand{\EE}{\mathbb{E}}

\DeclareMathOperator*{\argmin}{arg\,min}
\DeclareMathOperator*{\argmax}{arg\,max}

\newcommand{\Fcal}{\mathcal{F}}

\newcommand{\Gcal}{\mathcal{G}}

\newcommand{\Scal}{\mathcal{S}}

\newcommand{\Acal}{\mathcal{A}}
\newcommand{\Tcal}{\mathcal{T}}

\newcommand{\Ecal}{\mathcal{E}}

\newcommand{\Rmax}{R_{\max}}

\newcommand{\Vmax}{V_{\max}}
\newcommand{\RR}{\mathbb{R}}

\newcommand{\emp}[1]{\widehat{#1}}

\newcommand{\epsF}{\epsilon_{\text{approx}}}
\newcommand{\epsphi}{\epsilon_\phi}

\newcommand{\epsdct}{\epsilon_{\text{dct}}}
\newcommand{\epst}{\tilde{\epsilon}}
\newcommand{\epsd}{\epsilon_1}

\newcommand{\eTG}{\emp{\Tcal}_{\Gcal}^\pi}

\newcommand{\CS}{C_{\Scal}}
\newcommand{\CA}{C_{\Acal}}
\newcommand{\ts}{\tilde{s}}
\newcommand{\ta}{\tilde{a}}

\newcommand{\fQ}{\hat{Q}}

\newcommand{\fo}{f_0}

\newcommand{\bvftloss}{\texttt{BVFT-loss}\xspace}
\newcommand{\bvftpeloss}{\texttt{BVFT-PE-loss}\xspace}
\newcommand{\bvft}{\texttt{BVFT}\xspace}
\newcommand{\bvftpe}{\texttt{BVFT-PE}\xspace}
\newcommand{\bvftpeq}{\texttt{BVFT-PE-Q}\xspace}

\usepackage[style=alphabetic,backend=bibtex,maxcitenames=6,maxalphanames=6,maxbibnames=99]{biblatex}
\title{Towards Hyperparameter-free Policy Selection\\ for Offline Reinforcement Learning}

\author{%
Siyuan Zhang\\
Computer Science\\
University of Illinois at Urbana-Champaign\\
\texttt{siyuan3@illinois.edu} \\
\And
Nan Jiang \\
Computer Science\\
University of Illinois at Urbana-Champaign \\
\texttt{nanjiang@illinois.edu}
}

\begin{document}

\maketitle

\begin{abstract}
How to select between policies and value functions produced by different training algorithms in offline reinforcement learning (RL)---which is crucial for hyperparameter tuning---is an important open question. Existing approaches based on off-policy evaluation (OPE) often 
require additional function approximation and hence hyperparameters, creating a chicken-and-egg situation.  In this paper, we design  hyperparameter-free algorithms for policy selection based on \bvft \cite{xie2020batch}, a recent theoretical advance in value-function selection,  
and demonstrate their effectiveness in discrete-action benchmarks such as Atari. 
To address performance degradation due to poor critics in continuous-action domains, we further combine \bvft with OPE to get the best of both worlds, and obtain a hyperparameter-tuning method  for $Q$-function based OPE with theoretical guarantees  as a side product. 
\end{abstract}

\section{Introduction and Related Works}
\label{sec:intro}
Learning a good policy from historical data without interactive access to the actual environment, or offline (batch) reinforcement learning (RL), is a promising approach to applying RL to real-world scenarios when high-fidelity simulators are  not available \cite{levine2020offline}. Despite the fast development in the training algorithms, a burning  question that remains wide open is how to tune their hyperparameters, sometimes known as the offline policy selection problem \cite{paine2020hyperparameter, yang2020offline, fu2021benchmarks}. 

Standard approaches reduce the problem to off-policy evaluation (OPE), which estimates the expected return of the candidate policies  and choose accordingly. Unfortunately, OPE itself is a difficult problem, and standard estimators such as importance sampling suffer exponential (in horizon) variance \cite{li2015minimax, jiang2016doubly}. While polynomial-variance estimators exist, either using TD (e.g., Fitted-Q Evaluation, or FQE \cite{le2019batch}) or marginalized importance sampling \cite{liu2018breaking, nachum2019dualdice, uehara2020minimax}, they require  additional function approximation, inducing yet another set of hyperparameters (e.g., the neural-net architecture) which need to be carefully chosen. \cite{paine2020hyperparameter} recently conclude that FQE can be effective for offline policy selection, but ``an important remaining challenge is how to choose hyperparameters for FQE''. (Incidentally, we are able to address this question as a side product of our approach in Section~\ref{sec:bvftpe}.) In other words, to tune hyperparameters for training we need to tune the hyperparameters for OPE, creating a chicken-and-egg situation. To this end, we want to ask: 
\begin{center}
\textbf{Can we design effective \textit{hyperparameter-free} methods for offline policy selection?}
\end{center}
The question has been investigated in the theoretical literature \cite{farahmand2011model}, mostly reformulated so that we select indirectly among value functions instead of policies to trade-off directness for tractability. More precisely, \cite{farahmand2011model} imagines that training algorithms produce candidate $Q$-functions $Q_1, Q_2, \ldots, Q_m$, which is a reasonable assumption as most offline algorithms produce value functions as a side product. The goal is to select $Q_i \approx Q^\star$---assuming one exists---so that the induced greedy policy, $\pi_{Q_i}$, is near-optimal. While $\|Q-Q^\star\|$ is only a surrogate for the performance of $\pi_Q$, the hope is that whether $Q \approx Q^\star$ can be more easily verified from holdout data without additional function approximation, possibly by estimating the Bellman error (or residual) $\|Q-  \Tcal Q\|$. 
Unfortunately, $\|Q - \Tcal Q\|$ is not amendable to statistical estimation in stochastic environments \cite{sutton2018reinforcement}. The na\"ive estimator which squares the TD error (see ``1-sample BR'' in Proposition~\ref{prop:1sample}) suffers the infamous \textit{double-sampling bias} \cite{baird1995residual}, and debiasing approaches demand additional function approximation (and hence hyperparameters) \cite{antos2008learning, farahmand2011model}. In  prototypical real-world applications, hyperparameter-free heuristics such as picking the highest $Q$  \cite{garg2020batch} are often used despite the lack of theoretical guarantees, which we will compare to in our experiments. 

In this paper, we attack the problem based on a recent theoretical breakthrough in value-function selection: \cite{xie2020batch} propose a theoretical algorithm, \bvft, which provides a workaround to the double sampling issue without requiring additional function approximation; they estimate a form of projected Bellman error as a surrogate for $\|Q-Q^\star\|$, where the function class for projection is created out of the candidate $Q$'s themselves. See Section~\ref{sec:bvft_background} for details. Our contributions are 2-fold: \vspace{-.5em}
\begin{enumerate}[leftmargin=*, itemsep=0pt, labelindent=0pt]
\item We design a practical implementation of \bvft based on novel theoretical observations, removing its last hyperparameter which determines a discretization resolution. We empirically demonstrate that \bvft enjoys promising performance in discrete-action benchmarks such as Atari games, sometimes using 20x less data than required by FQE-based policy selection. 
\item The vanilla \bvft suffers performance degradation in continuous-action benchmarks, where the training algorithms often have an actor-critic structure and output a $(\pi, Q)$ pair where $Q$ is far away from $Q^\star$ for various reasons. 
To address this challenge, we propose \bvftpe, a variant of \bvft that allows us to select among $(\pi, Q)$ pairs and pick one where $Q \approx Q^\pi$ and $\pi$ yields a high return. To further handle the issue that $Q$ from the critic is often a poor fit of $Q^\pi$, we propose to use \textit{multiple} OPE algorithms to re-fit $Q^\pi$, and run \bvftpe among the produced $(\pi, Q)$ pairs. While the OPE algorithms often have many hyperparameters that need to be set, \bvftpe automatically chooses between them, leaving very few to no hyperparameters untunable. This allows us to combine the strengths of OPE and \bvft and get the best of both worlds. We also show additional results that \bvftpe can be used for hyperparamter tuning in $Q$-function-based OPE and provide theoretical guarantees, which is of independent interest. 
\end{enumerate}

\section{Preliminaries}
\label{sec:back}
\para{Markov Decision Processes (MDPs)}
In RL, we often model the environment as an MDP, specified by its state space $\Scal$, action space $\Acal$, reward function $R: \Scal\times\Acal\to[0, \Rmax]$, transition function $P: \Scal \times\Acal\to\Delta(\Scal)$ ($\Delta(\cdot)$ is the probability simplex), discount factor $\gamma \in [0, 1)$, and a initial state distribution $d_0$. 
We assume $\Scal\times\Acal$ is finite but can be arbitrarily large. 
A deterministic policy $\pi: \Scal\to \Acal$ 
induces a random trajectory $s_0, a_0, r_0, s_1, a_1, r_1, \ldots$ where $s_0 \sim d_0, a_t = \pi(s_t)$, $r_t = R(s_t, a_t)$, and $s_{t+1} \sim P(\cdot|s_t, a_t)$, $\forall t$. 
We measure the performance of $\pi$ using 
$J(\pi):= \EE[\sum_{t=0}^\infty \gamma^t r_t | \pi]$. 
In discounted MDPs, there always exists an optimal policy $\pi^\star$ that maximizes $J(\cdot)$ for all starting states. It is the greedy policy of the optimal $Q$-function, $Q^\star$, i.e., $\pi^\star = \pi_{Q^\star} := (s \mapsto \argmax_{a} Q^\star(s,a))$. $Q^\star$ is the fixed point of Bellman optimality equation, $Q^\star = \Tcal Q^\star$, where $\forall f \in \RR^{\Scal\times\Acal}$, $(\Tcal f)(s,a) := R(s,a) + \gamma \mathbb{E}_{s'\sim P(\cdot | s,a)}[\max_{a'} f(s', a')]$. A related important concept is 
$Q^\pi$,  
which tells us the expected return of $\pi$ when the trajectory starts from a specific state-action pair. 

\para{Offline Data} In offline RL, we are given a dataset of $(s,a,r,s')$ tuples and cannot directly interact with the MDP. For the theoretical part of the paper, we assume the standard \textit{offline sampling protocol}, that the tuples are generated i.i.d.~as $(s,a) \sim \mu$, $r = R(s,a)$, $s' \sim P(\cdot|s,a)$. 
With a slight abuse of notation, we also use $\EE_{\mu}[\cdot]$ to denote the expectation over $(s,a,r,s')$ sampled as above. 

\para{Policy Selection/Ranking} We will use the following unified framework for policy selection throughout the paper: Suppose training algorithms (or the same algorithm with different hyperparameters) produce multiple $(\pi, Q)$ pairs, $\{(\pi_i, Q_i)\}_{i=1}^m$, and our goal is to select a policy with good performance. The relationship between $\pi$ and $Q$ can differ in different contexts: for example, when training algorithms try to fit $Q^\star$ and induce a greedy policy, we have $\pi_i = \pi_{Q_i}$, and we only need to work with $\{Q_i\}_{i=1}^m$ as they contain all the relevant information. In the case where training algorithms have an actor-critic structure, $\pi$ and $Q$ are separate quantities and need to be reasoned about together. 
In real applications, the next step in the pipeline is to deploy the policy in the real system for online evaluation, and since we may have the resources to test more than 1 policy, we require all algorithms to produce a \textit{ranking} over $\{\pi_i\}_{i=1}^m$, often by sorting the policies in ascending order w.r.t.~a loss. 

\subsection{Experiment Setup}
\label{sec:exp_setup}
To avoid interrupting the flow of intertwined theoretical reasoning and empirical evaluation in the rest of the paper, we briefly describe our experiment setup here, with details deferred to Appendix~\ref{app:exp_setup}. 

\para{Environments and Datasets} 
We perform empirical evaluation on  OpenAI Gym \cite{openai}, Atari games \cite{bellemare2013arcade}, and Mujoco \cite{mujoco}. Taxi \cite{dietterich2000hierarchical} is used for sanity check. We use standard offline datasets when available (RLUnplugged \cite{gulcehre2021rl} for Atari, 
and D4RL \cite{fu2021d4rl} for MuJoCo), and generate our own otherwise by mixing a trained expert policy with 30\% chance of acting suboptimally.  
\cite{fu2021benchmarks} have  proposed a new benchmark for offline policy selection, which is also based on RLUnplugged/D4RL which we use and has only become available very recently. Also this benchmark (and \cite{voloshin2019empirical}) focuses on OPE and policy selection without value functions, which does not exactly fit our purposes. 
We leave the evaluation on their benchmarks to future work. 

\para{Training Algorithms}
We use several different training algorithms to generate the candidate models, including offline algorithms such as BCQ \cite{fujimoto2019benchmarking} and CQL \cite{kumar2020conservative}. To test the robustness of the policy-selection methods w.r.t.~how the candidate policies are generated, we also use algorithms that learn from online interactions such as DQN \cite{mnih2015human} 
in some domains, 
though policy selection is always performed using separate offline datasets. This scenario is also of interest in its own right, as one can imagine training on (possibly imperfect) simulators via online algorithms and using limited realworld offline data for policy selection. For each algorithm, we consider different neural architectures, learning rates, and learning steps as hyperparameters to produce multiple candidate policies (and value functions) for selection; see Table \ref{hyper_1} in Appendix~\ref{app:exp_setup} for details. 

\para{Performance Metrics}
In each experiment, we run different policy-selection methods and evaluate the produced policy rankings using the following two metrics adapted from \cite{yang2020offline}: \vskip 0em
\textit{\textbf{Top-$k$ normalized regret}}~ We take the best policy within the top-$k$ recommended by the ranking, and calculate its gap compared to the best policy among all candidates. This metric reflects our regret if we were to online-evaluate the top-$k$ policies in the ranking and identify the best among them. To make the value more interpretable, we normalize the regret by the gap between the best and the worst candidate policies. \\
\textit{\textbf{Top-$k$ precision}}~ We take the $k$ best policies in terms of their groudtruth values, and  return the proportion of them appearing in the top-$k$ policies in the ranking. \vskip 0em
This process is repeated for 200 or 300 runs, with randomness coming from re-sampling a subset of the dataset for policy selection (usually of size $50,000$; FQE needs much more data 
and we do not run it on random subsets) and sampling $m=10$ or $15$ policies from all candidates for comparison. All figures report the mean with error bars of twice the standard errors, i.e., $95\%$ confidence intervals.

  

\section{Background: Batch Value-Function Tournament (\bvft)}
\label{sec:bvft_background}
We briefly introduce the theoretical basis of our approach. As mentioned in Section~\ref{sec:intro},  the Bellman error $\|Q - \Tcal Q\|$ is an appealing quantity because $Q = Q^\star \Leftrightarrow \|Q - \Tcal Q\|_{\infty} = 0$, but it is not amendable to statistical estimation in stochastic environments due to the double-sampling bias \cite{baird1995residual, antos2008learning, farahmand2011model}. 
Given this caveat, a closely related quantity has been extensively studied in the literature: given function class $\Gcal \subset [0, \tfrac{\Rmax}{1-\gamma}]^{\Scal\times\Acal}$, the (mean-squared) \emph{projected Bellman error} of $Q$ w.r.t.~$\Gcal$ is defined as
\begin{align} \label{eq:mspbe}
\|Q - \Tcal_{\Gcal} Q\|_{2, \mu}^2, \textrm{~where~}  \Tcal_{\Gcal} Q := \argmin_{g\in\Gcal} \EE_{\mu}[(g(s,a) - r - \gamma \max_{a'} Q(s',a'))^2].
\end{align}
Here $\|\cdot\|_{2, \mu}^2 = \EE_\mu[(\cdot)^2]$, and the dependence of $\Tcal_\Gcal$ on $\mu$ is suppressed for readability. This quantity 
can be straightforwardly estimated from data when $\Gcal$ has bounded statistical complexity; we just need to replace all $\EE_{\mu}[\cdot]$ with their finite-sample approximation. 
The property of the projected Bellman error largely depends on the choice of $\Gcal$, which needs to satisfy certain conditions for $\|Q - \Tcal_\Gcal Q\|$ to be a good surrogate for $\|Q - Q^\star\|$. The seminal work of \cite{gordon1995stable} has provided the following sufficient condition: (see \cite[Proposition 4]{xie2020batch} for a formal proof) 
\begin{proposition} \label{prop:gordon}
If (1) $\Gcal$ is a piecewise-constant class, and (2) $Q^\star \in \Gcal$, then $\Tcal_\Gcal$ is $\gamma$-contraction under $\|\cdot\|_{\infty}$, and $Q=Q^\star \Leftrightarrow \|Q - \Tcal_\Gcal Q\|_{2, \mu} = 0$ if $\mu$ is fully supported on $\Scal\times\Acal$. 
\end{proposition}
Being piecewise constant means that 
there exists a partitioning of $\Scal\times\Acal$, and $\Gcal$ consists of all members of $[0, \tfrac{\Rmax}{1-\gamma}]^{\Scal\times\Acal}$ that remain constant within each partition. 
A partitioning whose induced $\Gcal$ satisfies (1) and (2) is also closely related to $Q^\star$-preserving 
state abstractions \cite{li2006towards}.

The problem is that it is very difficult to find $\Gcal$ that satisfies the above 2 criteria,\footnote{The ideal $\Gcal$ can be created similarly to Figure~\ref{fig:bvft_vis} according to knowledge of $Q^\star$. In Appendix~\ref{app:exp} we show that our approach closely tracks the skyline of using the ideal $\Gcal$ in a tabular domain where we can compute $Q^\star$. \label{ft:idealG}} and it will just be another set of hyperparameters to tune if we leave the design of $\Gcal$ to the user.  \cite{xie2020batch} offers a resolution in the context of selecting $Q^\star$ from $\{Q_i\}_{i=1}^m$: consider the base case of $m=2$ where we only need to select between $Q_1$ and $Q_2$. If one of them is $Q^\star$ (which can be relaxed) but we do not know which, we can still create $\Gcal_{1,2}$ that satisfies both criteria of Proposition~\ref{prop:gordon} as the minimal piecewise-constant  class such that $\{Q_1, Q_2\} \subset \Gcal_{1,2}$. If the output of each $Q_i$ takes at most $N$ possible values,  $\Gcal_{1,2}$ will be induced by partitioning $\Scal\times\Acal$ into at most $N^2$ regions; see Figure~\ref{fig:bvft_vis}L.\footnote{A common misconception is that \bvft only works under the ``assumption'' that $Q^\star$ is piecewise constant under certain given partition or metric over $\Scal\times\Acal$; this is not the case. The concept of piecewise-constant functions is used as an \textit{internal mechanism} in \bvft and is not an assumption in any way. The existence of $\Gcal_{1,2}$ is \textit{unconditional} and does not rely on any structural properties of the underlying MDP. Constructing the partition (and hence $\Gcal_{1,2}$) only requires $Q_1$ and $Q_2$, and does not require any additional prior knowledge (including but not limited to a pre-defined metric or partition over $\Scal\times\Acal$).}
\cite{xie2020batch} further shows that this idea extends to arbitrary $m$ via pairwise comparison (``tournament''): The validity of the procedure can be justified by the following simplified result:\footnote{Among all the $(i,j)$ pairs considered in $\bvftloss$, many of them will not satisfy $Q^\star \in \{Q_i, Q_j\}$ and thus the previous reasoning for the $m=2$ case and Proposition~\ref{prop:gordon} do not apply. Despite this, these $(i,j)$ pairs do not affect the validity of the $\bvftloss$; see \cite{xie2020batch} for a detailed explanation.}
\begin{proposition}[Simplification of \cite{xie2020batch}; see Appendix~\ref{app:bvftproof} for a proof sketch] \label{prop:bvft}
If $Q^\star \in \{Q_i\}_{i=1}^m$ and $\mu$ is fully supported on $\Scal\times\Acal$, then $Q_i = Q^\star \Leftrightarrow$ $Q_i$ having $0$ $\bvftloss$, where 
\begin{align} \label{eq:bvftloss}
\bvftloss(Q_i \,; \{Q_j\}_{j=1}^m) := \max_{j} \|Q_i - \Tcal_{\Gcal_{i,j}} Q_i \|_{2, \mu}.
\end{align}
\end{proposition} 
For each $Q_i$, \bvft calculates its projected Bellman error w.r.t.~$\Gcal_{i,j}$---where $\Gcal_{i,j}$ is created just like $\Gcal_{1,2}$ but using $Q_i$ itself and every other $Q_j$---and scores $Q_i$ using the worst-case projected error. 
See \cite{xie2020batch} for the complete theory that accounts for $Q^\star \notin \{Q_i\}_{i=1}^m$ and finite-sample effects. For readability we will stick to the above simplified reasoning.   

\para{Computation} The empirical version of \bvftloss can be computed exactly in closed form. The central step is the calculation of $\Tcal_\Gcal Q$. Recall that $\Gcal$ is piecewise constant and induced by a $\Scal\times\Acal$ partitioning. 
For any $(\tilde{s}, \tilde{a})$,  $(\Tcal_{\Gcal}Q)(\tilde{s},\tilde{a})$ is simply the average of $r + \gamma \max_{a'} Q(s',a')$ over all data points $(s,a,r,s')$ where $(s,a)$ falls in the same partition as $(\tilde{s}, \tilde{a})$. Implemented with running averages, the computational complexity is $O(m^2 n)$ for $m$ candidate functions and $n$ data points in addition to $(|\Acal|+1)mn$ candidate-function evaluations (i.e., caching $Q_i(s,a)$ and $\max_{a'}Q_i(s', a')$).

\begin{figure}[t]
\centering
\includegraphics[scale=0.5, trim=50 570 200 0, clip]{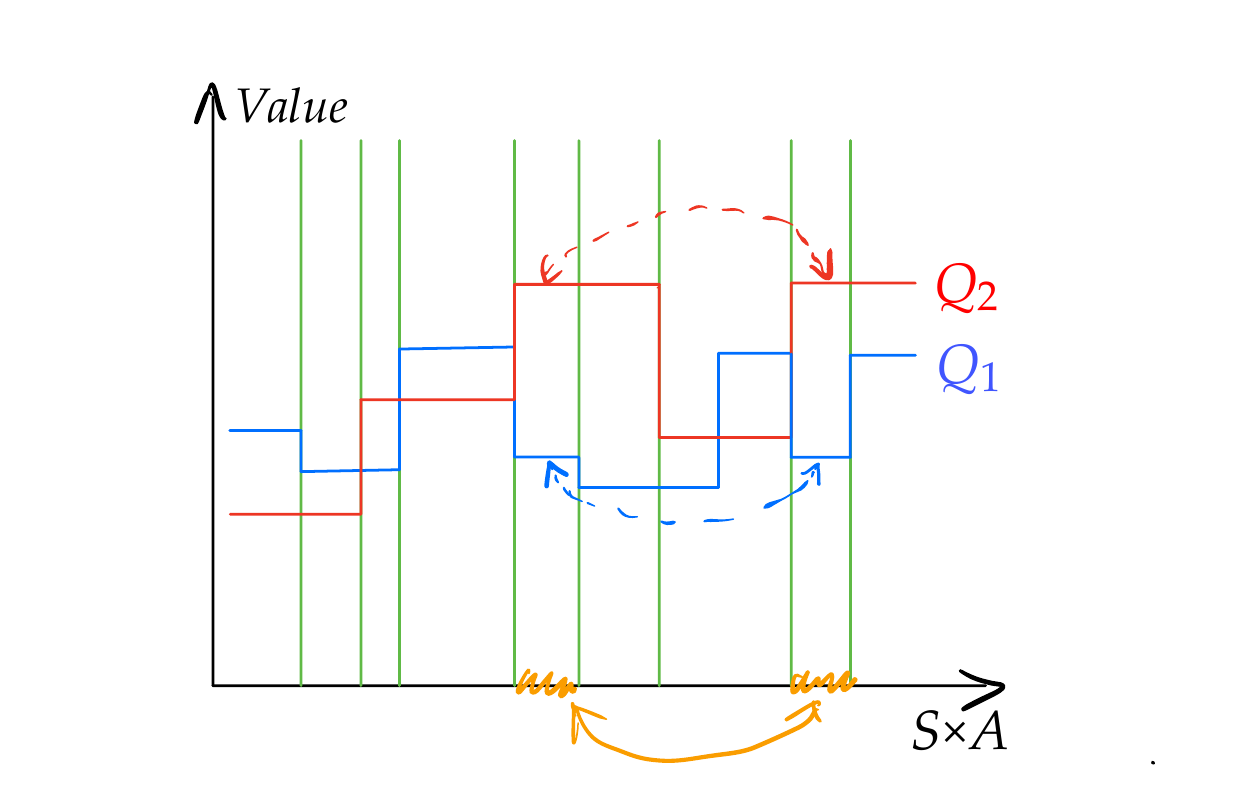} \quad
\includegraphics[width=0.33\columnwidth, trim=0 0 0 300]{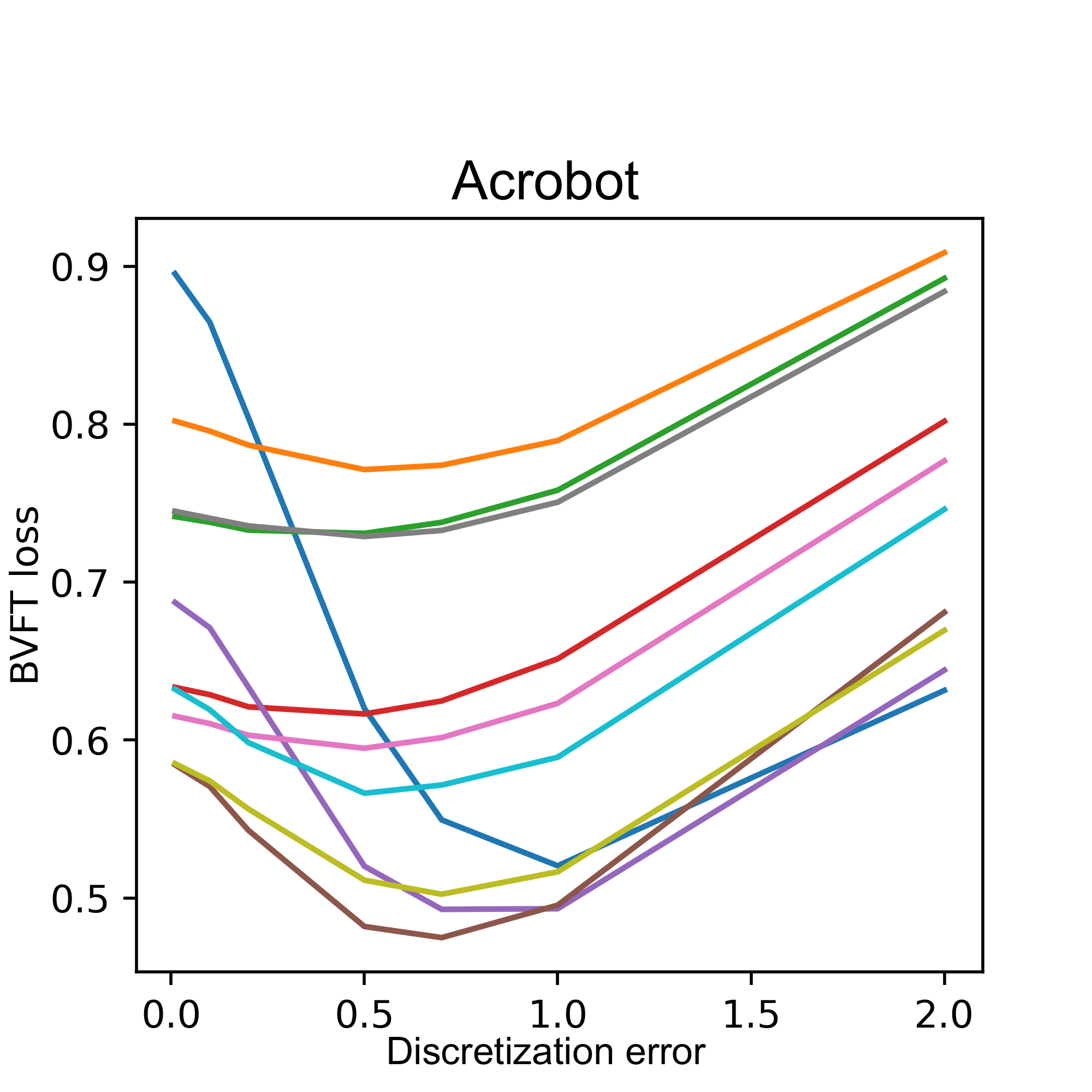}
\caption{\textbf{Left:} Visualization of the partitioning that induces $\Gcal_{1,2}$. The 2 subsets of $\Scal\times\Acal$ marked orange belong to the same partition despite being separated  from each other, 
because $Q_1$ and $Q_2$ are constant across them. 
Since $\Scal\times\Acal$ is partitioned according to $Q$'s \textit{output}, the number of partitions is 
\textit{independent} of $|\Scal\times\Acal|$, allowing \bvft to scale to arbitrarily complex state-action spaces with limited data. 
\textbf{Right:} \bvftloss vs.~discretization error $\epsdct$ of 10 candidate $Q$'s of a typical run in Acrobot, all having the U-shape predicted by our theoretical reasoning. \label{fig:bvft_vis}}
\end{figure}

\section{\bvft with Automatic Resolution Selection}
\label{sec:bvft}
Despite the appealing theoretical properties, it is unclear if \bvft can be converted into a practical algorithm. The sample complexity of estimating \bvftloss depends on the complexity of $\Gcal_{i,j}$, which is controlled by $N^2$ where $N$ is the number of possible values in $[0, \tfrac{\Rmax}{1-\gamma}]$ each $Q_i$ can take. To handle the issue that $N$ can be very large or even infinite, \cite{xie2020batch} discretizes the output of $\{Q_i\}_{i=1}^m$ up to $\epsdct$ error before producing $\{\Gcal_{i,j}\}$, where $\epsdct$ needs to be carefully chosen to trade-off between discretization errors and the complexity of $\Gcal$ (which is now $O(1/\epsdct^2)$). The theoretical value of $\epsdct$ (as in \cite[Theorem 2]{xie2020batch}) not only relies on unknown properties of the MDP and the data, but is also very small which makes \bvftloss expensive to estimate. Indeed, the previous theory predicts that \bvftloss becomes an unstable statistic when $\epsdct\to 0$ due to the unbounded complexity of $\Gcal_{i,j}$. 

To resolve this issue, we draw an interesting connection between \bvftloss and the na\"ive ``1-sample'' estimator for $\|Q-\Tcal Q\|$, which reveals the unexpected behavior of $\bvftloss$ in the regime of $\epsdct \to 0$ that differs from the previous theoretical predictions: 
\begin{proposition} \label{prop:1sample}
Consider $\EE_{\mu}[(Q(s,a) - r - \gamma \max_{a'} Q(s',a'))^2]$, the na\"ive and biased estimator for $\|Q-\Tcal Q\|_{2, \mu}^2$ which we call ``1-sample BR (Bellman residual)''. If each candidate $Q_i$ never predicts the same value for any two $(s,a)$ pairs seen in the dataset, then with $\epsdct=0$, the empirial version of $\bvftloss(Q_i; \{Q_j\}_{j=1}^m)$ and 1-sample BR of $Q_i$ coincide (up to squaring). 
\end{proposition}
\begin{proof}
Given a dataset $D$ consisting of $(s,a,r,s')$ tuples, the empirical version of 1-sample BR is $\frac{1}{|D|}\sum_{(s,a,r,s') \in D}(Q(s,a)-r-\gamma \max_{a'} Q(s',a'))^2$, and that of \bvftloss squared is $\max_{j} \frac{1}{|D|}\sum_{(s,a,r,s') \in D}(Q(s,a)- (\emp{\Tcal}_{\Gcal_{i,j}} Q)(s,a))^2$, where $\emp{\Tcal}_{\Gcal_{i,j}}$ is the empirical version of $\Tcal_{\Gcal_{i,j}}$ based on the same dataset. It suffices to show that for any data point $(s,a,r,s')$, $(\emp{\Tcal}_{\Gcal_{i,j}} Q)(s,a) =  r + \gamma \max_{a'} Q_i(s',a')$, which follows immediately from $\epsdct=0$ and $Q_i$ never predicting the exact same value twice, since $\Gcal_{i,j}$ does not provide any aggregation over data points in this case and $\emp\Tcal_{\Gcal_{i,j}}$ coincides with the 1-sample Bellman update (c.f.~the paragraph on computation in Section~\ref{sec:bvft_background}). 
\end{proof}
This result implies that, 1-sample BR, which is a reasonable objective and coincides with $\|Q-\Tcal Q\|$ in deterministic environments, provides a safeguard to \bvftloss when $\epsdct$ is too small. 
Since the double-sampling bias of 1-sample BR is positive \cite{baird1995residual}, we expect \bvftloss to gradually decrease as $\epsdct$ increases, hit a minimum, and increase again due to discretization errors when $\epsdct$ becomes too large.\footnote{This can be violated in some extreme cases; see Appendix~\ref{app:alg} for details. \label{ft:largedct}} 
Indeed, this is precisely what we observe empirically; see Figure~\ref{fig:bvft_vis}R. 

Based on this novel observation, we propose to search for a grid of discretization errors in \bvft and pick the resolution that minimizes the loss (Eq.\eqref{eq:bvftloss}); see pseudocode in Appendix~\ref{app:alg}. In the experiments we will always use this rule to automatically select the discretization resolution for \bvft and its variants. The remaining questions can only be answered empirically: 
Does \bvft ever exhibit more interesting behavior than 1-sample BR (especially given their intimate relationship),  is our resolution selection rule a good one, and how does \bvft compare to other baselines? 

\begin{figure}[t]
\centering
\includegraphics[width=\columnwidth]{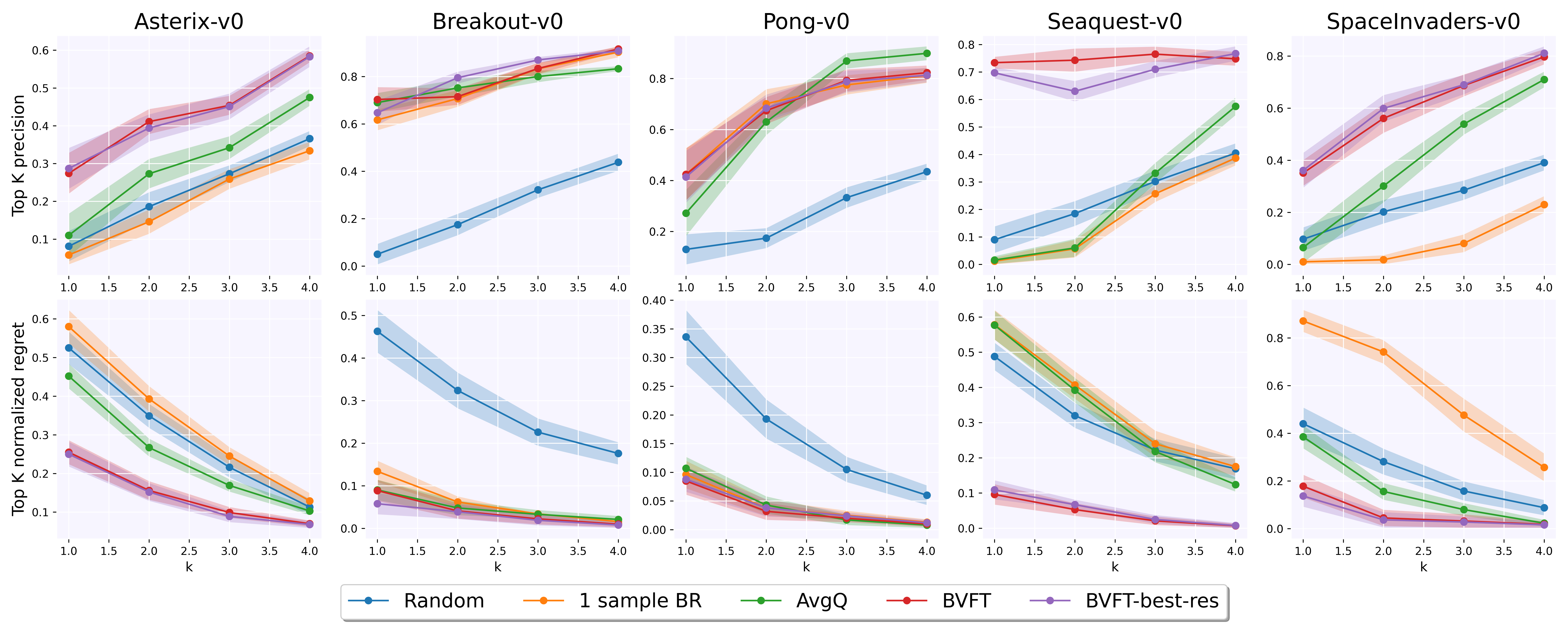}
\caption{Top-$k$ metrics of policy rankings vs.~$k$ in Atari. Row 1 shows top-$k$ precision (the higher the better), and Row 2 shows top-$k$ regret (the lower the better). Training algorithms are BCQ with different hyperparameters. The dataset for policy selection has 50,000 transition, which is an order of magnitude less than needed by FQE in Atari (see FQE in Enduro \cite{voloshin2019empirical}, as well as Figure~\ref{fig:bvft_vs_fqe}). \label{fig:atari-topk}}
\end{figure}

\para{Empirical Evaluation} To answer these questions, we empirically compare \bvft, 1-sample BR, and another simple hyperparameter-free heuristic, AvgQ \cite{garg2020batch}, which simply ranks $\{Q_i\}_{i=1}^m$ based on their average value on the data. ``Random'' ranks the policies in a completely random manner. \bvft and 1-sample BR rank $\{Q_i\}_{i=1}^m$ in ascending order of their loss functions, respectively. The experiments are done on 5 Atari games (Figure~\ref{fig:atari-topk}) and 4 Gym control problems (Figure~\ref{fig:control_topk}; action space is made discrete in Pendulum). 

\textit{\textbullet~ Does \bvft ever exhibit more interesting behavior than 1-sample BR?} Perhaps surprisingly, \bvft deviates significantly from the behavior of 1-sample BR, and almost always outperforms the latter by a wide margin. This is particularly interesting in CartPole and Pendulum, where the deterministic dynamics make 1-sample BR an \textbf{unbiased} estimate of $\|Q-\Tcal Q\|$, and we expected it to perform well. Contrary to our expectation, 1-sample BR performs poorly even in these domains, where \bvft often performs much better. 
In fact, we observe similar phenomenon in a version of Taxi where we 
can compute $\|Q-\Tcal Q\|$ as a skyline; see Appendix~\ref{app:exp}. 

\textit{\textbullet~ Is our resolution selection rule effective?} To examine the effectiveness of our resolution selection rule, we compare to a skyline of \bvft itself (``\bvft-best-res''), where the best fixed resolution is chosen 
based on average statistics in the hindsight. As we can see, \bvft with automatic resolution selection closely tracks the performance of \bvft-best-res---to the extent that they are mostly indistinguishable---indicating that our resolution selection rule is near-optimal. 

\textit{\textbullet~ Comparison to AvgQ.} The simple heuristic of picking $Q$ with the highest predicted value---which is sometimes used in prototypical applications \cite{garg2020batch}---can be surprisingly effective sometimes, such as in Pendulum, LunarLander, and Pong. \bvft performs equally well in these domains. Moreover, as results from other domains reveal, AvgQ is very unstable and can fail catastrophically, such as in Acrobot and Seaquest, and is much less robust compared to \bvft. 

One may wonder if AvgQ works better with \textit{pessimistic} training algorithms: if every $Q_i$ is an under-estimation of the true return, then maximizing $Q$ will be a well-justified heuristic. However, such pessimism is not always guaranteed especially when the function approximation is misspecified. In Appendix~\ref{app:exp} we show additional experiments in Atari with a pessimistic algorithm CQL, and the results are qualitatively similar to Figure~\ref{fig:atari-topk}.

\begin{figure}[t]
\centering
\includegraphics[width=0.89\columnwidth]{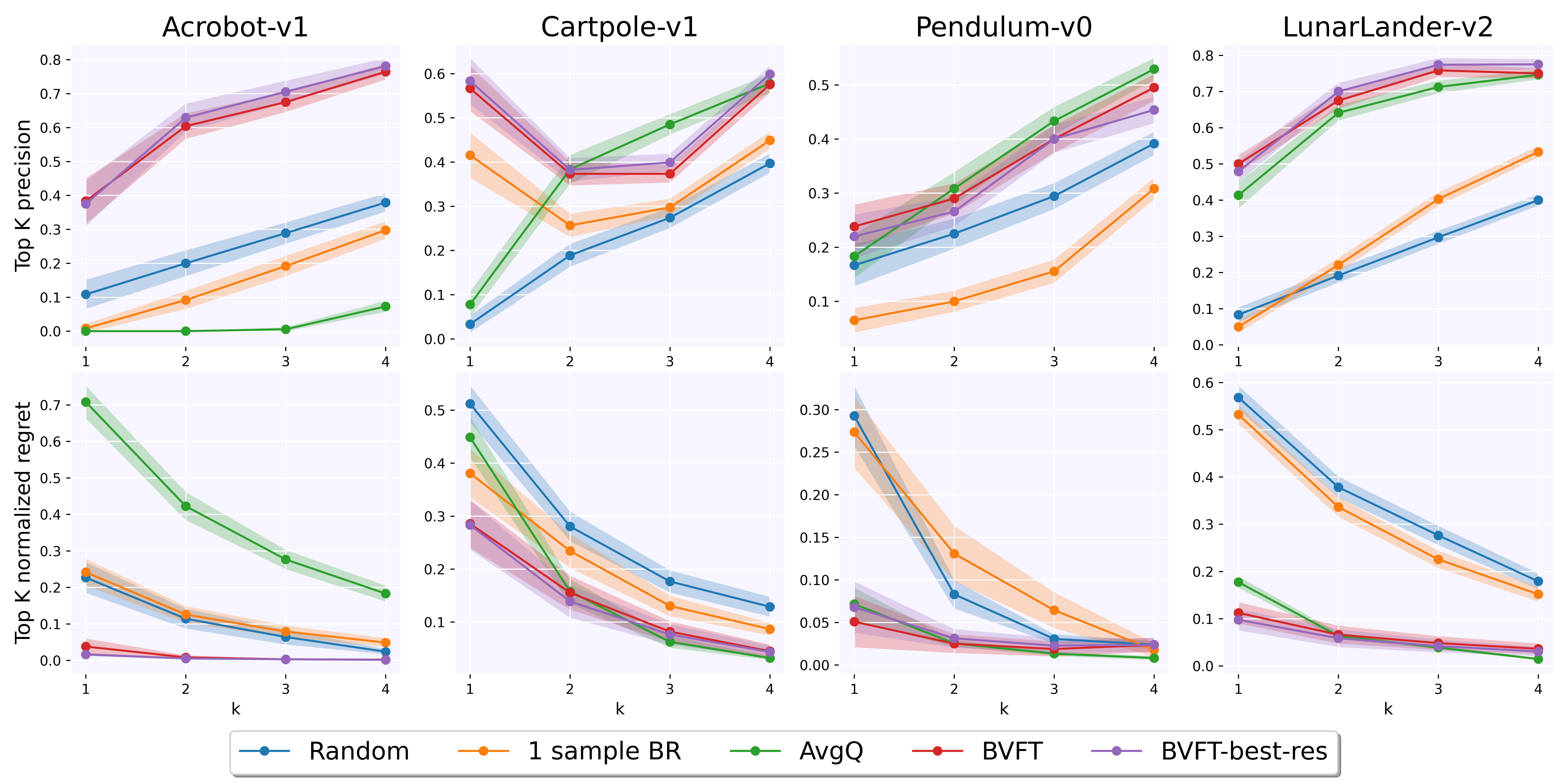}
\caption{Results in Gym control problems.  The dataset for policy selection has 50,000 transitions. \label{fig:control_topk}}
\end{figure}

\para{Comparison to OPE} Despite the lack of known method for tuning OPE's hyperparameters except for ``cheating'' in the simulator using online roll-outs (in Section~\ref{sec:bvftpe} we will combine \bvft with OPE to address this issue), which makes it very difficult to have fair comparisons with OPE-based policy-selection methods, it is still instructive to have a rough sense of how our method compares to OPE. 
In this and the next sections, we choose Fitted Q-Evaluation (FQE) as a representative OPE algorithm.\footnote{Despite recent exciting developments in marginalized importance sampling (MIS), FQE often shows strong empirical performance thanks to its simplicity \cite{voloshin2019empirical, paine2020hyperparameter, fu2021benchmarks}, whereas MIS is not as off-the-shelf due to its difficult optimization, so we defer the comparison as well as combining \bvft with MIS to future work.} 
Figure~\ref{fig:bvft_vs_fqe} shows the ranking metrics vs.~sample size in 5 Atari games. We tune the neural architecture for FQE by choosing the training architecture that produces the \textit{best} policy in Asterix.\footnote{The recent OPE benchmarks \cite{paine2020hyperparameter, fu2021benchmarks} do not include Atari, so we choose our own hyperparameters.} 
While FQE has equally good performance compared to \bvft in the large-sample regime ($10^6$ transitions, which is typically required for FQE in Atari), the performance degradation is severe when sample size decreases. In comparison, \bvft is much more sample-efficient and provides almost the same level of performance with about 20x less samples, which is also 
what we used in Figure~\ref{fig:atari-topk}. 

\para{Additional Results} In Appendix~\ref{app:exp} we examine the sensitivity of different methods to the dataset size and exploratoriness. \bvft remains advantageous compared to the baselines, is insensitive to data exploratoriness, and enjoys improved performance when more data becomes available unlike other hyperparmeter-free baselines. 

\begin{figure}[t]
\centering
\includegraphics[width=\columnwidth]{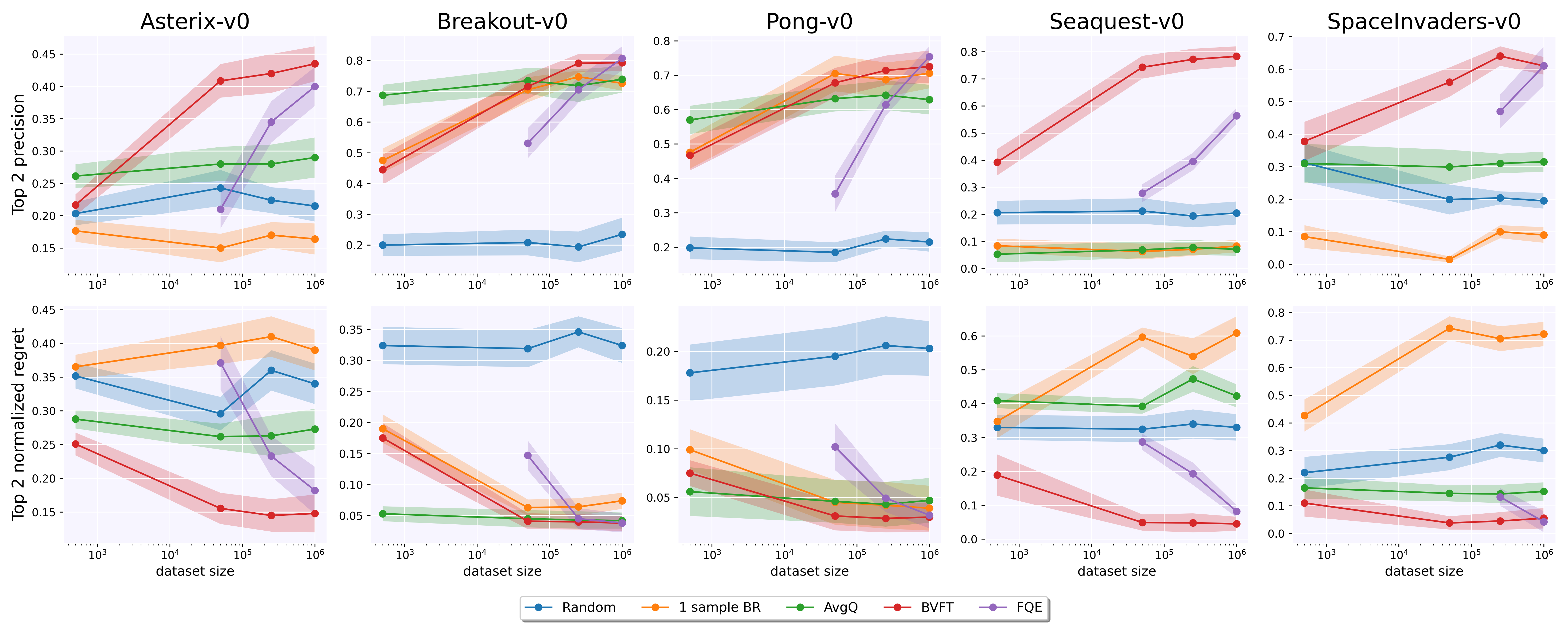}
\caption{Comparison to FQE on Atari games. The figures show top-2 metrics vs.~sample sizes. We did not test FQE on small sample sizes since the trend of performance degradation is clear. \label{fig:bvft_vs_fqe} }
\end{figure}

\section{Fighting Poor Value-Function Estimates with \bvft + OPE}
\label{sec:bvftpe}
The experiments so far are on discrete-action domains, and the candidate policies are always greedy w.r.t.~the $Q$-function. However, in continuous-action domains, actor-critic-type algorithms are often used, where a critic $Q$ is trained to evaluate a parameterized policy $\pi$, and the actor $\pi$ is in turn improved based on $Q$. The trained $(\pi, Q)$ pair in general does not satisfy $\pi = \pi_Q$, so it seems unwise to ignore $\pi$ and only focus on $Q$ in policy selection. On a related note, $\argmax_{a'}$ is often expensive if not infeasible to calculate in continuous action spaces, so we need to make changes to \bvft anyway. 

More importantly, the success of \bvft relies on the existence of a reasonable approximation of $Q^\star$ among the candidate functions, which does not always hold especially for actor-critic algorithms. Indeed, in the Mujoco experiments shown in Figure~\ref{fig:mujoco-topk}, we observe poor performance of all hyperparameter-free methods, including 
variants of \bvft for continuous-action domains we develop later, and no method can consistently outperform the trivial baseline of random ranking. 
We address this issue of poor critics in the rest of this section. The issue is quite complicated and has many contributing factors---as we will explain subsequently---requiring us to take a multi-step approach to address one factor at a time.

\begin{figure}[t]
\centering
\includegraphics[width=0.65\columnwidth]{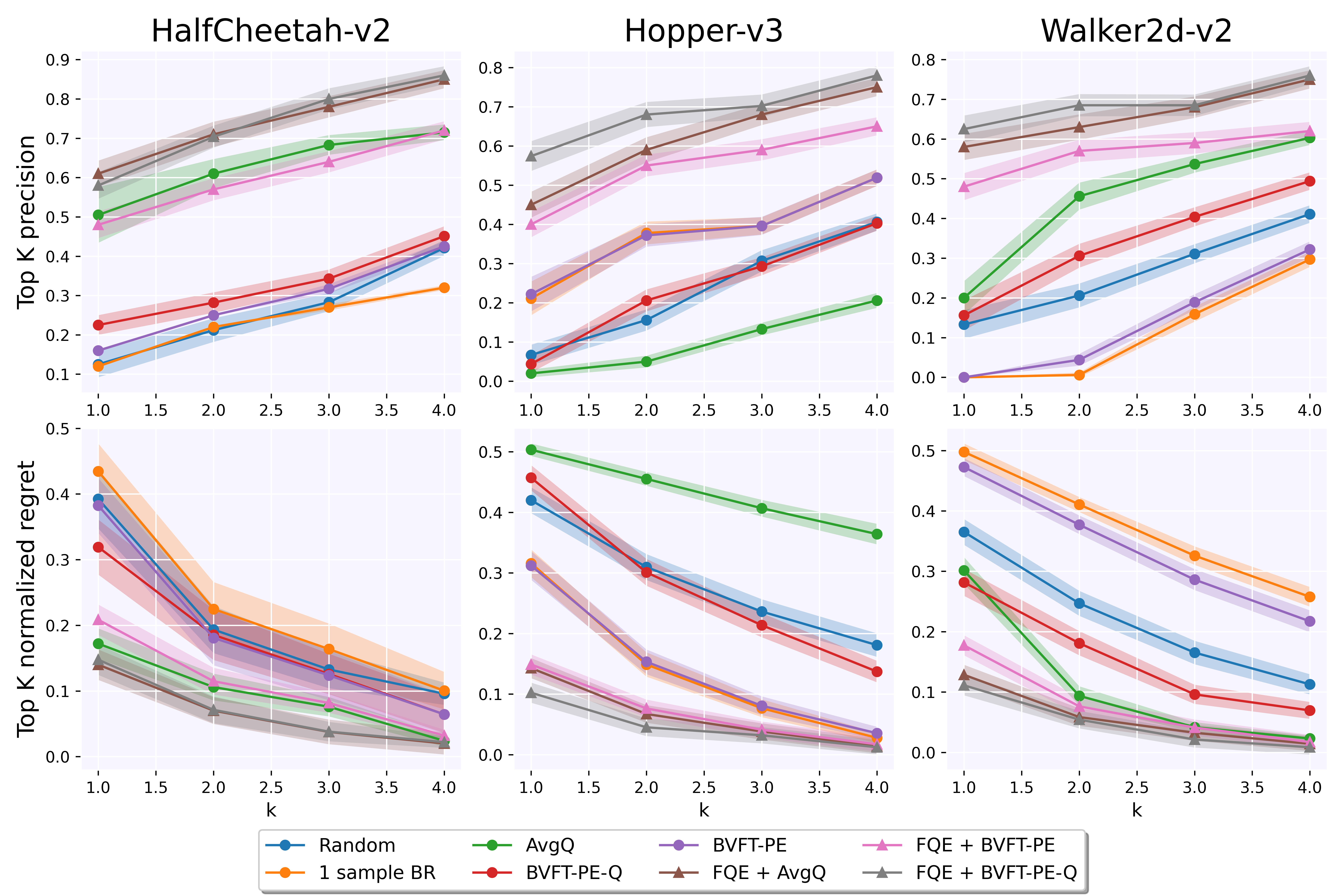}
\caption{Top-$k$ metrics vs.~$k$ across 3 Mujoco domains with continuous actions. All methods use a dataset of size $50,000$, except that the FQE component of any method uses $10^6$. 
\label{fig:mujoco-topk} }
\end{figure}

\para{Step 1: \bvftpe and \bvftpeq for Joint Selection of $(\pi, Q)$} We first address the issue that  $\argmax_{a'}$  may be infeasible to calculate in continuous-action domains, and that both $\pi$ and $Q$ need to be taken into consideration for actor-critic algorithms. To address this issue, we propose a variant of \bvft, called \bvftpe, with the following loss function:
\begin{align} \label{eq:bvftpeloss}
\bvftpeloss((\pi_i, Q_i) \,; \{(\pi_j, Q_j)\}_{j=1}^m) := \max_{j} \|Q_i - \Tcal^{\pi_i}_{\Gcal_{i,j}} Q_i \|_{2, \mu},
\end{align}
where $\Gcal_{i,j}$ is exactly the same as in \bvft, and $\Tcal^{\pi_i}_{\Gcal_{i,j}}$ is the same as Eq.\eqref{eq:mspbe}, except that $\max_{a'}Q(s',a')$ should be replaced by $Q(s', \pi(s'))$. (In actor-critic algorithms, $\pi$ is often a stochastic policy, in which case the term means $\EE_{a' \sim \pi(\cdot|s')} [Q(s', a')]$.) This loss function naturally generalizes the $\bvftloss$: when each $\pi_i$ is the greedy policy of $Q_i$, we immediately have $\Tcal^{\pi_i}_{\Gcal_{i,j}} Q_i = \Tcal_{\Gcal_{i,j}} Q_i$, thus recovering \bvftloss as a special case. 

A closer inspection of Eq.\eqref{eq:bvftpeloss}  reveals an interesting fact: $\bvftpeloss((\pi_i, Q_i)) = 0$ as long as $Q_i = Q^{\pi_i}$, so \bvftpeloss is really a loss function for \textit{policy evaluation}, hence the name \bvftpe. (We will actually provide the theoretical guarantees of $\bvftpe$ for policy evaluation at the end of this section; see Theorem~\ref{thm:bvftpe}.) 
While the loss recovers $\bvftloss$ when $\pi = \pi_{Q}$, more generally there can exist the degenerate cases of $Q = Q^{\pi}$ but $\pi$ itself being a poor policy, where $\bvftpeloss((\pi, Q))$ is still $0$. We address this issue by subtracting a $Q$ term from the $\bvftpeloss$: $\bvftpeq:= \bvftpeloss - \lambda \EE_{\mu}[Q]$, and the structure of this loss resembles a telescoping identity for $J(\pi)$ commonly used in the OPE literature; see Appendix~\ref{app:bvftpeq} for details.

\para{Step 2: Re-fitting $Q$ with OPE} The above derivations address some of the basic theoretical issues, but still implicitly assume that at least one good $\pi_i$ is paired with a $Q_i \approx Q^{\pi_i}$. In actor-critic algorithms, however, we often observe that the actor $\pi$ converges way before the critic $Q$ does, leaving us with a poorly estimated $Q$. Indeed, we have already seen in Figure~\ref{fig:mujoco-topk} that \bvftpe and \bvftpeq are still not effective when applied to the candidate $(\pi, Q)$ pairs. 

A natural solution to the problem is to use OPE algorithms to refit $Q^\pi$ to replace the critic $Q$, in order to provide a more accurate value function. We implement this using FQE as the OPE algorithm, though in principle we can use any OPE algorithm that provides an estimate of  $Q^\pi$, including kernel loss \cite{feng2019kernel}, MQL \cite{uehara2020minimax},  or even a model-based method through planning. Figure~\ref{fig:mujoco-topk} shows that \bvftpe and \bvftpeq enjoy strong empirical performance. However, the baseline of simply ranking the policies according to FQE itself (``FQE+AvgQ'') is equally effective, which suggests that most of the performance gains should be attributed to FQE. 
That said, even if one's conclusion is that ``OPE-based policy selection is more superior in continuous-action domains'', we are just all the way back to where we started in Section~\ref{sec:intro}: 
\begin{center}
\textit{How can we tune the hyperparameters for OPE itself?}
\end{center}

\begin{figure}[t]
  \centering
\includegraphics[width=0.63\columnwidth]{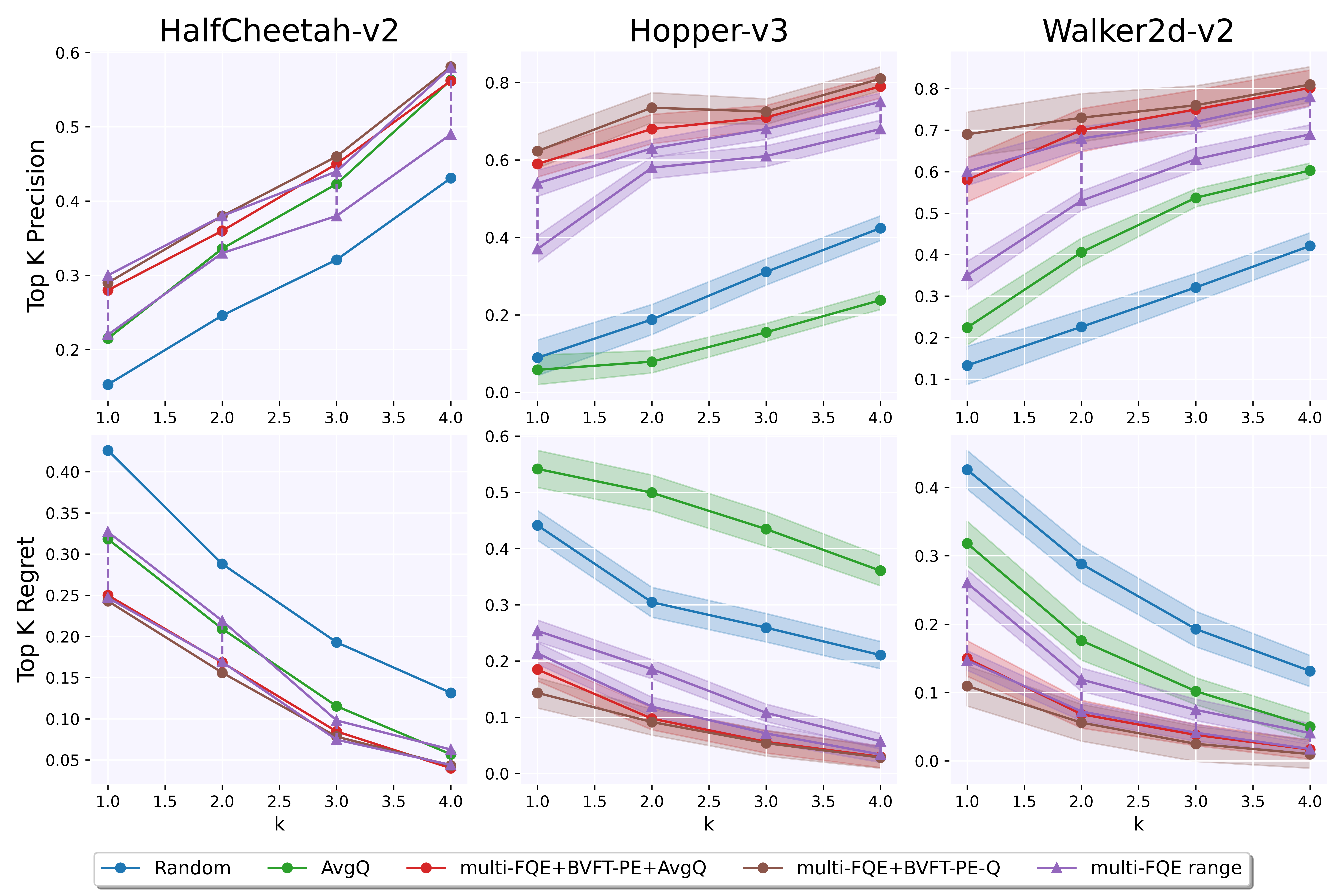}
\raisebox{0.1\height}{  \includegraphics[width=0.35\columnwidth]{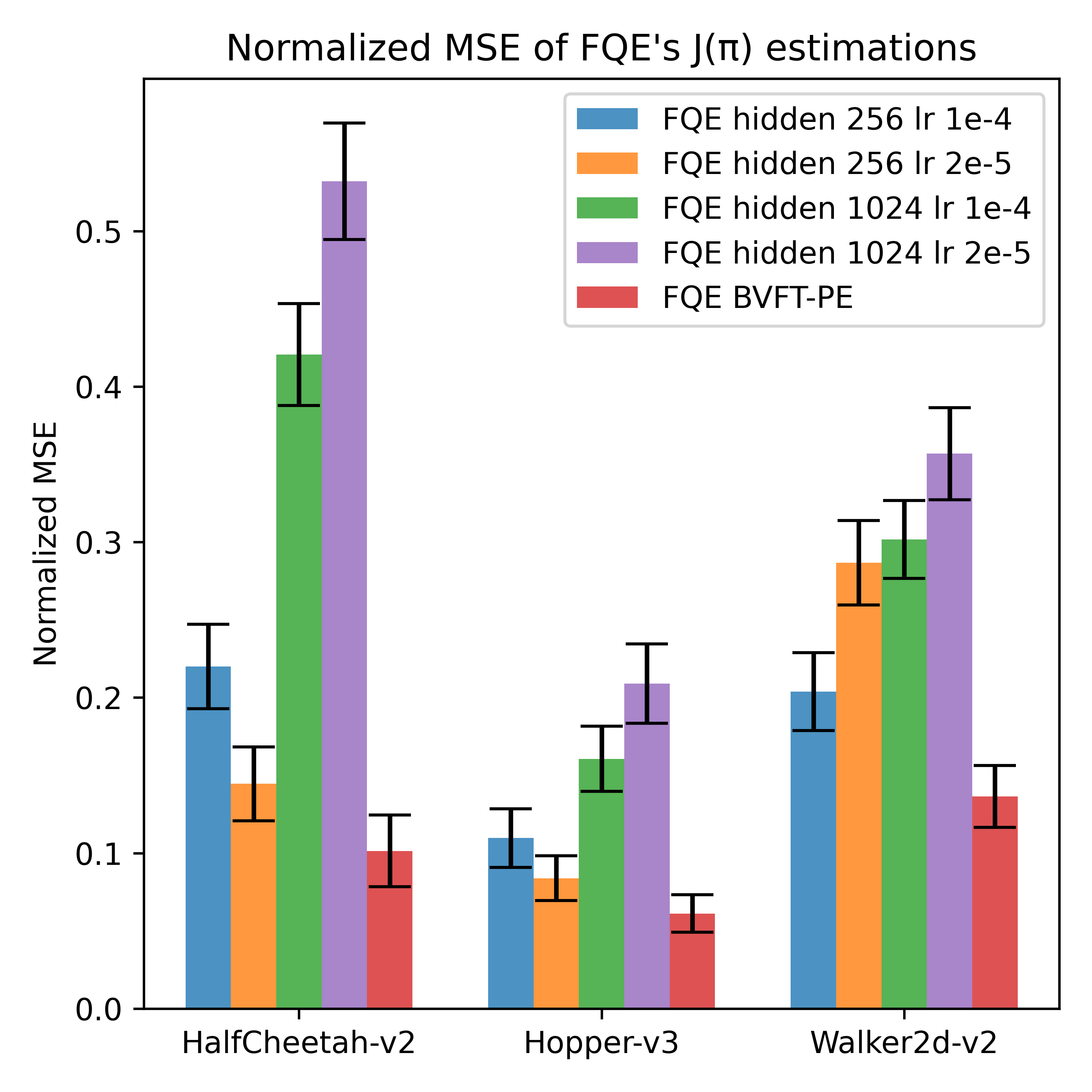}}
\caption{\textbf{Left:} Policy selection with multiple FQE instances in Mujoco. To avoid cluttering, we show the range of FQE performance across different hyperparameters (with upper and lower bounds connected by vertical dashed lines), and have removed error bars from HalfCheetah for better visibility (their widths are similar to the other two figures). The curves are cluttered in HalfCheetah and both our strategies perform similarly to the upper bound.  \textbf{Right:} OPE accuracy for using \bvftpe for hyperparameter tuning in FQE, where $J(\pi)$ is approximated by $\EE_{s \sim d_0} [Q(s, \pi(s))]$. \label{fig:multi-ope}} \end{figure}

\para{Step 3: Re-fitting $Q$ with \textit{Multiple} OPE Algorithms} 
Our last idea is retrospectively simple, yet allows us to combine the strengths of \bvft and OPE and get the best of both worlds. Suppose we have $L$ OPE algorithms which we wish to select from. We will simply run each OPE algorithm on each policy, producing $m\times L$ candidate $(\pi, Q)$ pairs in total, in the form of $\{(\pi_i, Q_i^l)\}$, where $Q_i^l$ is the estimation of $Q^{\pi_i}$ by the $l$-th OPE algorithm. There are two strategies we can proceed with:

\para{Strategy 1:} Run \bvftpeq on all $m\times L$ pairs of $(\pi, Q)$. Rank the $m$ policies according to their highest position in the ranking produced by $\bvftpeq$. 

\para{Strategy 2:} Within each $\pi_i$, use \bvftpe to select $Q_i^{l_i^\star} \approx Q^\pi$, then rank the policies based on the predictions of $Q_{i}^{l_i^\star}$ (e.g., using AvgQ).

\para{Applicability to Arbitrary Training Algorithms} We motivated \bvftpe and \bvftpeq by actor-critic training and assumed that training algorithms must produce $(\pi, Q)$ pairs. However, the approaches proposed above can be readily applied to any training algorithms: we discard the $Q$ from training in Step 2, and therefore the training algorithms do not need to produce value functions in the first place. 

\para{Empirical Evaluation} We use these strategies in the same experiment setting as Figure~\ref{fig:mujoco-topk}, with multiple instances of FQE using different hyperparameters (see legend of Figure~\ref{fig:multi-ope}R).\footnote{Here the training algorithms are BCQ. In Appendix~\ref{app:exp} we reproduce Figure~\ref{fig:multi-ope}L with CQL as the training algorithms; see Figure~\ref{fig:mujoco-cql}.  The results are qualitatively the same.}  
The results are shown in Figure~\ref{fig:multi-ope}L, where both strategies perform similarly to or better than the best FQE instance. We emphasize that this is highly nontrivial, because \bvftpe and \bvftpeq do not observe the \textit{identities} of the FQE algorithms across different policies and runs (in fact, each FQE algorithm is represented as merely $2mn$ numbers), so matching the best FQE performance is strong evidence for \bvft algorithms' model selection capabilities. Between the two \bvft strategies,  Strategy 1 (using \bvftpeq) slightly outperforms Strategy 2, but comes with an additional hyperparameter $\lambda$; we tuned it on Hopper and use the same constant in all experiments. 
In comparison, Strategy 2 is only slightly worse and does not have its own hyperparameters, which is an advantage. 

\para{Hyperparameter Tuning in OPE} In Strategy 2, we are basically using OPE as the policy-selection algorithm, and \bvftpe as a subroutine for hyperparameter tuning for OPE itself. Since OPE is an important component of the offline RL pipeline in its own right, our procedure for tuning OPE algorithms using \bvftpe is also of independent interest. To this end, we conduct additional experiments in Mujoco to test the OPE accuracy of FQE with different hyperparameters. As Figure~\ref{fig:multi-ope}R shows, FQE with hyperparameters tuned by \bvftpe consistently outperforms the \textit{best fixed} set of hyperparameters across all 3 Mujoco domains. This implies that \bvftpe can select the best hyperparameters for each policy individually, which is an appealing property. We also supplement the empirical results with a theoretical guarantee for selecting $Q^\pi$ out of candidate functions for a fixed $\pi$, which follows from similar proof techniques as \cite[Theorem 2]{xie2020batch}; see Appendix~\ref{app:theory} for the proof. 

\begin{theorem} \label{thm:bvftpe}
Let $C$ be the same as in \cite[Theorem 2]{xie2020batch}, which characterizes the exploratoriness of the data distribution. Consider any policy $\pi$ and candidate $Q$-functions, $\{Q^{l}\}_{l=1}^L$, with $Q^l \in [0, \tfrac{\Rmax}{1-\gamma}]$. Let $(\pi, \hat{Q})$ be the pair that minimizes the empirical version of \bvftpeloss applied to $\{(\pi, Q^l)\}_{l=1}^L$ with $\epsdct = \frac{\epsilon\Rmax}{8\sqrt{C}}$. Then, using a dataset of size $\tilde O\left(\frac{C^2 \ln (L/\delta)}{\epsilon^4 (1-\gamma)^4}\right)$  where $\tilde{O}$ suppresses logarithmic terms, w.p.~$\ge 1-\delta$, 
$\sup_{\nu: \|\nu/\mu\|_\infty \le C} \|\fQ - Q^\pi\|_{2, \nu} \le \epsilon \cdot \frac{\Rmax}{1-\gamma} + \frac{(2 + 4\sqrt{C})\min_{l}\|Q^l - Q^\pi\|_\infty}{1-\gamma}$.
\end{theorem}
As the theorem implies, when one of $\{Q^l\}_{l=1}^L$ is a good approximation of $Q^\pi$ (i.e., $\min_l \|Q^l -Q^\pi\|_{\infty}$ is small), \bvftpe is able to identify $\hat{Q} \approx Q^\pi$ 
with a polynomial sample complexity. 


\section{Discussion and Conclusion}
\label{sec:discuss}
We present \bvft and its variants based on recent theoretical advances \cite{xie2020batch}, which are (nearly) hyperparameter-free algorithms for policy selection in offline RL and empirically effective in discrete-action benchmarks. When combined with OPE algorithms such as FQE, variants of \bvft are also competitive in continuous-action benchmarks, and such a combination addresses the weaknesses of \bvft (relying on the existence of good value functions among the candidates) and those of OPE (having untunable hyperparameters) and gets the best of both worlds. 

We conclude the paper with discussions of the limitations of our approach and open questions:

\para{Sample Efficiency} Section~\ref{sec:bvftpe} uses OPE to fit $Q^\pi$, which can be data intensive. A plausible solution is to re-use the training data for OPE, as suggested by \cite{paine2020hyperparameter}. 
However, data reuse can cause serious issues in realworld domains where the  amount of training data taken for granted in deep RL is not available. 
How to improve the sample efficiency of \bvft + OPE is an important open question.

\para{Computational Complexity} The $O(m^2)$ complexity of \bvft can be prohibitive if we wish to compare hundreds of models. A plausible solution is divide-and-conquer, i.e., eliminating bad models by running \bvft in smaller groups. 
It will be interesting to see if this compromises performance. 

\para{Applicability}
It is important to understand what type of domains \bvft is particularly suited for. Despite having provided partial answers (e.g., the discussion of discrete actions vs.~continuous actions), we still need a more thorough answer based on comprehensive evaluation across more diverse domains and comparison to a wider range of baselines, which is beyond the scope of this paper since we focus on algorithm development and some of our contributions are orthogonal to existing approaches (e.g., \bvft + OPE). We look forward to investigating this question empirically in the future with the help of the public benchmarks that have become available recently \cite{fu2021benchmarks}.

\para{Data with Insufficient Coverage} While results in Appendix~\ref{app:exp} show that \bvft is insensitive to moderate changes in data exploratoriness, we will likely need to incorporate some form of pessimism--which is shown to be important for training \cite{liu2019off, liu2020provably, kidambi2020morel, jin2020pessimism, rashidinejad2021bridging, xie2021bellman}---when the data coverage is seriously lacking. This can be, however, quite challenging, as \bvft uses a dynamic function space for projection ($\Gcal_{i,j}$) that varies from candidate to candidate, and states considered covered due to generalization effects under one function space may be considered lacking data under a different one. How to resolve this issue and design a pessimistic version of \bvft is an interesting question.




\begin{ack}
Nan Jiang acknowledges funding support from the ARL Cooperative Agreement W911NF-17-2-0196, NSF IIS-2112471, and Adobe Data Science Research Award.
\end{ack}

\printbibliography 

\newpage 
\appendix
\label{appendix}
\section{Algorithm Details} \label{app:alg}
\begin{algorithm}[H]
\SetAlgoLined
\KwIn{Dataset $D$, candidate functions $\{Q_i\}_{i=1}^m$, discretization parameter set $\mathcal{R} = \{\epsilon_{k}\}$}
\For{$\epsilon_k \in \mathcal{R}$}{
\For{$i=1$ to $m$}{
$\overline{Q_i} \leftarrow$ discretize the output of $Q_i$ with resolution $\epsilon_{k}$.
}
\For{$i =1$ to $m$}{
 \For{$j = 1$ to $m$}{
  Define $\phi_{i,j}$ as a partitioning of $\Scal\times\Acal$ s.t.~$\phi(s,a) = \phi (s', a')$ iff $\overline{Q_i}(s, a) = \overline{Q_i}(s', a')$ and $\overline{Q_j}(s, a) = \overline{Q_j}(s', a')$\;
  Let $\mathcal{G}_{i,j}$ be the piecewise constant function class induced by $\phi_{i,j}$\;
  $\hat{\mathcal{T}}_{\Gcal_{i,j}} Q := \argmin_{g\in \mathcal{G}_{i,j}} \frac{1}{|D|} \sum_{(s,a,r,s')\in D} [(g(s,a) - r - \gamma \max_{a'} Q(s', a'))^2]$ \;
  $\mathcal{E}_{\epsilon_k}(Q_i ; Q_j) \leftarrow \|Q_i - \hat{\mathcal{T}}_{\Gcal_{i,j}} Q_i\|_{2,D}$ \quad {\color{blue} \# $\|f\|_{2, D}^2 := \frac{1}{|D|} \sum_{(s,a)\in D} f(s,a)^2$} \;
 }
 $\Ecal_{\epsilon_k}(Q_i) \leftarrow \max_{j}\, \mathcal{E}_{\epsilon_k}(Q_i; Q_j)$\;
 }
 }
 \KwOut{Sort $\{Q_i\}_{i=1}^m$ in ascending order according to $\min_{k} \Ecal_{\epsilon_k}(Q_i)$} 
 \caption{\bvft with Automated Resolution Selection \label{alg:bvft}}
\end{algorithm}
\para{Remark on Footnote~\ref{ft:largedct}} As Footnote~\ref{ft:largedct} mentioned, there is a corner case where our resolution selection rule fails: when $\epsdct = \infty$, a constant $Q$ that predicts $\EE_{D}[r]/(1-\gamma)$ can have $0$ \bvftloss. This can be prevented in theory by adding a penalty term $ \propto \epsdct$, though in practice we never observe such a degenerate behavior  and hence did not include the penalty in our implementation. 

\subsection{Details of \bvftpeq} \label{app:bvftpeq}
Here we provide the details about \bvftpeq. As mentioned in Section~\ref{sec:bvftpe},  when $\pi_i \ne \pi_{Q_i}$, it is possible that $Q_i = Q^{\pi_i}$ but $\pi_i$ itself is a poor policy, in which case $\bvftpeloss((\pi_i, Q_i))$ is still $0$. We address this issue by taking inspiration from a telescoping identity commonly used in the OPE literature \cite{uehara2020minimax}: $\forall (\pi, Q)$, let $d^\pi$ be the normalized discounted state-action occupancy, then
$$J(\pi) = \EE_{s\sim d_0} [Q(s, \pi(s))] - \tfrac{1}{1-\gamma}\EE_{d^\pi}[Q - \Tcal^\pi Q].$$
This tells us that we can use \textit{any} $Q$ to evaluate $\pi$, as long as we penalize its prediction $\EE_{s\sim d_0} [Q(s, \pi(s))]$ with the Bellman error of $Q$ w.r.t.~$\pi$. Inspired by this, we consider $\bvftpeq((\pi_i, Q_i); \{(\pi_j, Q_j)\}_{j=1}^m) = \bvftpeloss - \lambda \EE_{\mu}[Q_i]$, where \bvftpeloss is a form of Bellman error w.r.t.~$\pi$, and $\EE_{\mu}[Q_i]$ (which is what AvgQ uses to rank the $Q_i$'s) is a variant of $\EE_{s\sim d_0} [Q(s, \pi(s))]$. We are ranking policies in the descending order according to the loss, which is why the signs are the opposite to the telescoping identity. Intuitively, this avoids the degenerate issue because when $Q = Q^\pi$ for a poor $\pi$, $\bvftpeq((\pi_i, Q_i)) = - \lambda \EE_{\mu}[Q^\pi]$; instead of receiving a $0$ loss as in \bvftpe, the policy is still penalized for having a low $Q^\pi$. This new loss is effectively a linear combination of AvgQ and \bvftpe, and $\lambda$ is an additional hyperparameter that determines their relative contributions. 

\section{Proofs} \label{app:theory}
\subsection{Proof Sketch of Proposition~\ref{prop:bvft}} \label{app:bvftproof}
The basic idea is that $\Gcal_{i,j}$ is piecewise-constant by design, and 
satisfies $Q^\star \in \Gcal_{i,j}$ when $i^\star \in \{i,j\}$. Based on Proposition~\ref{prop:gordon}, this implies that whenever $i^\star \in \{i,j\}$, $Q=Q^\star \Leftrightarrow \|Q - \Tcal_{\Gcal_{i,j}} Q \|_{2, \mu}=0$. Then, to show that $Q_i = Q^\star \Leftrightarrow \bvftloss(Q_i; \{Q_j\}_{j=1}^m)=0$, it suffices to show that $Q_i=Q^\star \Rightarrow \bvftloss(Q_i) = 0$ and $Q_i\ne Q^\star \Rightarrow \bvftloss(Q_i) \ne 0$ :
\begin{itemize}[leftmargin=*]
\item When $Q_i= Q^\star$, $Q^\star \in \Gcal_{i,j}$ for any $j$, so $\bvftloss(Q_i) = \bvftloss(Q^\star) = 0$.
\item When $Q_i \ne Q^{\star}$, \begin{align*}
\bvftloss(Q_i) = &~ \max_{j} \|Q_i - \Tcal_{\Gcal_{i,j}} Q_i\|_{2, \mu} \\
\ge &~ \|Q_i - \Tcal_{\Gcal_{i,i^\star}} Q_i\|_{2, \mu} \ne 0 \tag{$Q^\star \in \Gcal_{i, i^\star}$}.
\end{align*}
\end{itemize}

\subsection{Proof of Theorem~\ref{thm:bvftpe}}

In this section we provide a full proof of Theorem~\ref{thm:bvftpe}. We start by quoting \cite{xie2020batch}'s Assumption 1 on the data exploratoriness, which Theorem~\ref{thm:bvftpe} relies on. We will also be able to reuse a number of their lemmas and propositions, due to the insensitivity of those results to the difference between $Q^\star$ and $Q^\pi$. That is, the proof holds literally by replacing any $\max_{a'}Q(s',a')$ term with $Q(s', \pi(s'))$ (or $\EE_{a'\sim \pi(\cdot|s')}[Q(s', a')]$ if $\pi$ is stochastic), $Q^\star$ with $Q^\pi$, and $V^\star$ with $V^\pi$, in which case we say the proof ``translates from $Q^\star$ to $Q^\pi$''. Also note that their function class $\Fcal$ corresponds to our $\{Q^l\}_{l=1}^L$. 

\begin{assumption} \label{asm:con}
We assume that $\mu(s,a) > 0~ \forall s,a$. 
We further assume that \\
(1) There exists constant $1 \le \CA< \infty$ such that for any $s\in\Scal, a\in\Acal$, $\mu(a|s) \ge 1/\CA$. \\
(2) There exists constant $1 \le \CS < \infty$ such that  for any $s\in\Scal, a\in\Acal, s'\in\Scal$, $P(s'|s,a)  / \mu(s') \le \CS$. Also $d_0(s) / \mu(s) \le \CS$. \\
It will be  convenient to define $C = \CS \CA$. 
\end{assumption}

\begin{lemma}[Lemma 3 of \cite{xie2020batch} translated to $Q^\pi$] \label{lem:Mphi}
Let $\phi$ be a partitioning of $\Scal\times\Acal$, such that $\phi(s,a) = \phi(\ts, \ta)$ means that $(s,a)$ and $(\ts, \ta)$ are in the same partition. Define $M_\phi = (\Scal, \Acal, P_\phi, R_\phi, \gamma, d_0)$, where  
\begin{align*}
& R_{\phi}(s,a) = \frac{\sum_{\ts, \ta: \phi(\ts, \ta) = \phi(s,a)} \mu(\ts, \ta) R(\ts, \ta)}{\sum_{\ts, \ta: \phi(\ts, \ta) = \phi(s,a)} \mu(\ts, \ta) }. \\
& P_{\phi}(s'|s,a) = \frac{\sum_{\ts, \ta: \phi(\ts, \ta) = \phi(s,a)} \mu(\ts, \ta) P(s'|\ts, \ta)}{\sum_{\ts, \ta: \phi(\ts, \ta) = \phi(s,a)} \mu(\ts, \ta) }.
\end{align*}
Then $\Tcal_\Gcal^\pi$ for the piecewise-constant class $\Gcal$ induced by $\phi$ is the Bellman operator for policy $\pi$ in MDP $M_\phi$. 
\end{lemma}

\begin{proposition}[Proposition 1 and Lemma 6 of \cite{xie2020batch}]
Let $\nu$ be a distribution over $\Scal\times\Acal$ and $\pi$ be a policy. Let $\nu' = P(\nu) \times \pi$ denote the distribution specified by the generative process $(s',a') \sim \nu' \Leftrightarrow (s,a) \sim \nu, s'\sim P(\cdot|s,a), a' = \pi(s')$. Under Assumption~\ref{asm:con}, we have $\|\nu'/\mu\|_\infty := \max_{s, a} \nu'(s,a)/\mu(s,a) \le C$.  Also note that $\|(d_0 \times \pi) / \mu\|_\infty \le C$. Furthermore, the same is true when the MDP dynamics is replaced by $M_\phi$ for any $\phi$. 
\end{proposition}

\begin{proposition}[Proposition 5 of \cite{xie2020batch} translated to $Q^\pi$] \label{prop:phi}
Given any $\phi$, define $\epsphi := \min_{g\in\Gcal} \|g - Q^\pi\|_\infty$, where $\Gcal$ is the piecewise-constant class induced by $\phi$. 
Fixing any $\epsd, \epst$. Suppose the dataset $D$ has size
\begin{align}\label{eq:sample_size_phi}
|D| \ge \frac{32\Vmax^2 |\phi|\ln\frac{8\Vmax}{\epst \delta}}{\epst^2} + \frac{50 \Vmax^2 |\phi| \ln{\frac{80 \Vmax}{\epsd \delta}}}{\epsd^2},
\end{align}
where $|\phi|$ is the number of partitions in $\phi$ and $\Vmax = \Rmax/(1-\gamma)$. 
Then, with probability at least $1-\delta$, 
for any $\nu\in \Delta(\Scal\times\Acal)$ such that $\|\nu/\mu\|_{\infty} \le C$,  
\begin{align} \label{eq:nobad}
\textstyle \|\fo - Q^\pi \|_{2,\nu} \le \frac{2\epsphi + \sqrt{C}(\|\fo - \eTG \fo \|_{2, D} + \epsd + \epst)}{1-\gamma},
\end{align}
where $\eTG$ is the empirical estimation of $\Tcal_\Gcal$ based on dataset $D$. 
At the same time, 
\begin{align} \label{eq:hasgood}
\|\fo - \eTG \fo\|_{2, D} \le (1+\gamma)\|\fo - Q^\star\|_\infty + 2\epsphi + \epst + \epsd. 
\end{align}
\end{proposition}
\begin{proof}
This is a central result of \cite{xie2020batch}'s proof and we show that the entire proof translates to $Q^\pi$. In particular, the result relies on a number of lemmas in \cite{xie2020batch}, among which 
\begin{itemize}[leftmargin=*, itemsep=0pt]
\item Lemmas 6 (quoted above), 8, and 10 (two concentration bounds) are independent of $Q^\star$ or $Q^\pi$ and can be used without modification.
\item Lemma 9 provides concentration bound for $\emp\Tcal_\Gcal f$ for a fixed $f$, and only uses the boundedness of $V_f(s') := \max_{a'} f(s',a')$. The same bound holds for $\eTG$ when we replace $V_f$ with $f(s', \pi(s'))$.
\end{itemize}
Finally, the proof of the proposition itself translates when we replace $\pi_{f,f'}$ in their proof with simply $\pi$. This is because error propagation in policy evaluation is much simpler than that in policy optimization \cite{farahmand2010error, chen2019information}, so we only need to be concerned with the distributions generated by $\pi$ instead of various other policies  (see e.g., Appendix I of \cite{uehara2021finite}). \end{proof}

The final lemma we need is:
\begin{lemma}[Lemma 11 of \cite{xie2020batch} translated to $Q^\pi$]
Let $\phi$ be the partitioning of $\Scal\times\Acal$ that induces  $\Gcal_{i,j}$ in \bvftpe, satisfies $|\phi| \le (\Vmax/\epsdct)^2$. Let $l^\star := \argmin_{l} \|Q^l - Q^\pi\|_\infty$. When $l^\star \in \{i, j\}$, we further have $\epsphi \le \epsdct + \min_{l} \|Q^l - Q^\pi\|_\infty$.
\end{lemma}

\begin{proof}[\textbf{Proof of Theorem~\ref{thm:bvftpe}}]
The majority of the proof of \cite{xie2020batch}'s Theorem 2 translates, except for the final part where they use $\|\hat Q - Q^\star\|$ (our $\hat Q$ is their $\hat f$) to bound the suboptimality of the greedy policy of $\hat Q$. However, since we are only concerned about policy evaluation, we are not interested in this part that does not translate. What is useful to us is an intermediate result that translates: when the sample size is set to 
\begin{align} \label{eq:intermediate_sample_size}
|D| \ge \frac{82\Vmax^4 \ln\frac{160\Vmax L}{\epst \delta}}{\epst^2 \epsdct^2},
\end{align}
with probability at least $1-\delta$, for any $\nu$ s.t.~$\|\nu/\mu\|_\infty \le C$,
\begin{align*}
\|\fQ - Q^\pi \|_{2,\nu} 
\le &~ \frac{(2+4\sqrt{C})\epsF + 4\sqrt{C}(\epsdct + \epst)}{1-\gamma}.
\end{align*}
To guarantee that $\frac{4\sqrt{C}(\epsdct + \epst)}{1-\gamma} \le \epsilon \Vmax$, we set $\epsdct = \epst = \frac{\epsilon\Rmax}{8\sqrt{C}}$. Plugging this back into Eq.\eqref{eq:intermediate_sample_size} yields the sample complexity in the theorem statement.
\end{proof}

\section{Experiment Setup Details}
\label{app:exp_setup}

\subsection{Environments}

\paragraph{Atari Games}
Atari games (or formally the Arcade Learning Environment \cite{bellemare2013arcade}) are a set of arcade games used for benchmarking RL algorithms. The observations are raw-pixels and the action space is finite. 
We select 5 (Pong, Breakout, Asterix, Seaquest, and Space Invader) commonly used environments for our experiments based on their different characteristics. The pixel observation is resized to an 84 $\times$ 84 image per frame and with a frameskip of 4 and sticky action (25\% of chance that the previous action will be executed instead of agent's action). Pong and Breakout have deterministic transition dynamics while the other 3 environments are more stochastic. In order to extract useful features from the high dimensional pixel input, we used multiple layers of convolutional neural networks followed by fully connected layers as our function approximator, which is the standard architecture in this domain \cite{mnih2015human}.

\paragraph{Classic Control and Box2D}
OpenAI gym \cite{openai} classic control provides many classical control problems. Among them, Cartpole and Acrobot have low dimensional (4, 6) continuous state space with discrete action space of cardinalities 2 and 3. Pendulum originally has a continuous action space but was modified to have 2 actions of swinging to left and right to match the other environments. 
All the classic control environments have deterministic transition dynamics. 

We also experimented with LunarLander, a Box2D environment with higher complexity and difficulty than the classic control environments. The state space is continuous with discrete actions of firing engines at different locations to control the descent of the lunar lander. The episode stops when the lunarlander's pads touch the surface and rewards are given based on fuel consumption and landing location. We used simple 3-layer MLP to be the function approximator which is sufficient for the relatively low input dimension. The environment has random starting locations and velocities while the transition dynamics are deterministic.

\paragraph{Mujoco}
The gym MuJoCo has a collection of continuous control tasks implemented based on the MuJoCo Simulator \cite{mujoco} and has been a popular testing environment for continuous-action RL algorithms. We included HalfCheetah, Hopper, and Walker2D to evaluate the performance of BVFT and other baselines.

\paragraph{Taxi}
Taxi \cite{dietterich2000hierarchical} is a classical reinforcement learning environment with a 2D grid world that simulates taxi driving along with the grids. At each time step, the taxi can choose among 6 actions of moving north, south, west, east, or stay and picks up or drops off the passenger. The RL agent will receive a reward of 20 if the taxi successfully picks up or drops of a passenger, otherwise the agent receives a reward of -1 at each time step. The original taxi environment has a grid size of 5 $\times$ 5, resulting in 500 total states (25 $\times$ 4 $\times$ 5, corresponds to 25 possible taxi locations, 4 destination locations, and 5 passenger locations(4 spawn locations and 1 in the taxi)). In order to further investigate the effect of randomness in transition dynamics, we performed additional experiments on a  modified version of the original taxi environments with additional stochasticity: we replaced the only original passenger with 4 passengers randomly spawning and disappearing at the 4 locations every time step. The resulting environment has a total of 2000 states (25 $\times$ $2^4$ $\times$ 5, corresponds to 25 possible taxi locations, $2^4$ passenger appearance status, and 5 taxi status). In addition, we added a parameter $p_{rand}$ to the environment such that at each time step, the taxi has a probability of $p_{rand}$ that it will act randomly instead of following the agent's action.

A main reason for considering this simple environment is the easy access to $\Tcal$ and $Q^\star$ due to the tabular nature of the environment. 
This allows us to introduce a number of skylines to evaluate \bvft's performance, which would not be possible in more complex environments. We use tabular Q-learning as training algorithms to generate the candidate $Q$'s. 
The results of this study can be found in Figure~\ref{taxi-topk}.

\subsection{Training algorithms and datasets}
BCQ \cite{ fujimoto2019benchmarking} and CQL \cite{kumar2020conservative} are used to learn $Q$-functions in Atari from the RLUnplugged datasets \cite{gulcehre2021rl}. (In Appendix~\ref{app:exp} we also use online DQN to generate candidate policies in Atari.) 
For the classic control and Box2D domains, we deployed DQN \cite{mnih2015human} with 2/3 layer MLP as function approximators to learn the candidate model in an online manner. The offline dataset for policy selection is generated by first training an expert policy, and then mixing 70\% of expert trajectories with 30\%  $\epsilon$-greedy trajectories w.r.t.~the expert for $\epsilon=0.5$. 
In MoJoCo, policies are learned using the continuous-action version of BCQ \cite{fujimoto2019offpolicy} using D4RL datasets \cite{fu2021d4rl}. 
In taxi, standard Q-learning is used to learn $Q^*$ online. See Table~\ref{hyper_1} for further details about the candidate hyperparameters of each training algorithm in each domain.

\begin{table}[t]
\vskip 0.2in
\begin{center}
\begin{tabular}{@{}lll@{}}
\toprule
Hyperparameter      & Atari                                                  & MuJoCo                \\ \midrule
Training size & 1M & 500K \\
Hidden size          & FC: 256, 1024; CNN: (32,64,3136)                       & 64, 1024              \\
\# hidden layers & FC: 1, 2                                               & 2, 3                  \\
Learning rate       & 0.0000625, 0.00001                                     & 0.001,  0.00001       \\
Learning steps       & \{200k:100k:1M\}/\{50k:50k:400k\}/\{50k:50k:400k\}  & \{50k:25k:300k\} \\
Algorithms          & DQN/BCQ/CQL        & BCQ/CQL            \\ \bottomrule
\end{tabular}
\vskip 0.05in
\begin{tabular}{@{}llll@{}}
\toprule
Hyperparameter      & Taxi                       & Gym Control       & Box2d                  \\ \midrule
Training size &  Online & Online & Online\\
Hidden size          & Tabular                    & 64, 256, 1024         & 64, 256, 1024          \\
\# hidden layers & Tabular                    & 2, 3                  & 2, 3                   \\
Learning rate       & \{5e-3:5e-3:2.5e-2\} & 2.5e-4, 5e-4       & 2.5e-4, 5e-4        \\
Learning steps       & \{200k:50k:500k\}     & \{50k:25k:250k\} & \{100k:50k:400k\} \\
Algorithms          & Q-learning                 & DQN                   & DQN                    \\ \bottomrule
\end{tabular}
\vskip 0.05in
\caption{Hyperparameters of the candidate models. We consider an array of hyperparameters on the optimizer and the neural architecture to produce a diverse set of candidate models for selection. The notation $\{a:x:b\}$ is a shorthand for $\{a, a+x, a+2x, \ldots, b\}$.}
\label{hyper_1}
 \end{center}
\vskip -0.2in
\end{table}

\section{Additional Experiments}
\label{app:exp}
\para{Atari with CQL and DQN as the training algorithms} Figure~\ref{fig:atari-cql} shows the results in Atari when we use CQL as the training algorithm (still using offline data), and Figure~\ref{fig:atari-dqn} shows what happens when DQN learns candidate policies using online interactions with the environment. Both results are qualitatively similar to Figure~\ref{fig:atari-topk}, where BCQ is used as the training algorithm. This demonstrates the robustness of \bvft w.r.t.~different training algorithms, or even offline vs.~online training. 

\begin{figure}
  \centering
\centerline{\includegraphics[width=\columnwidth]{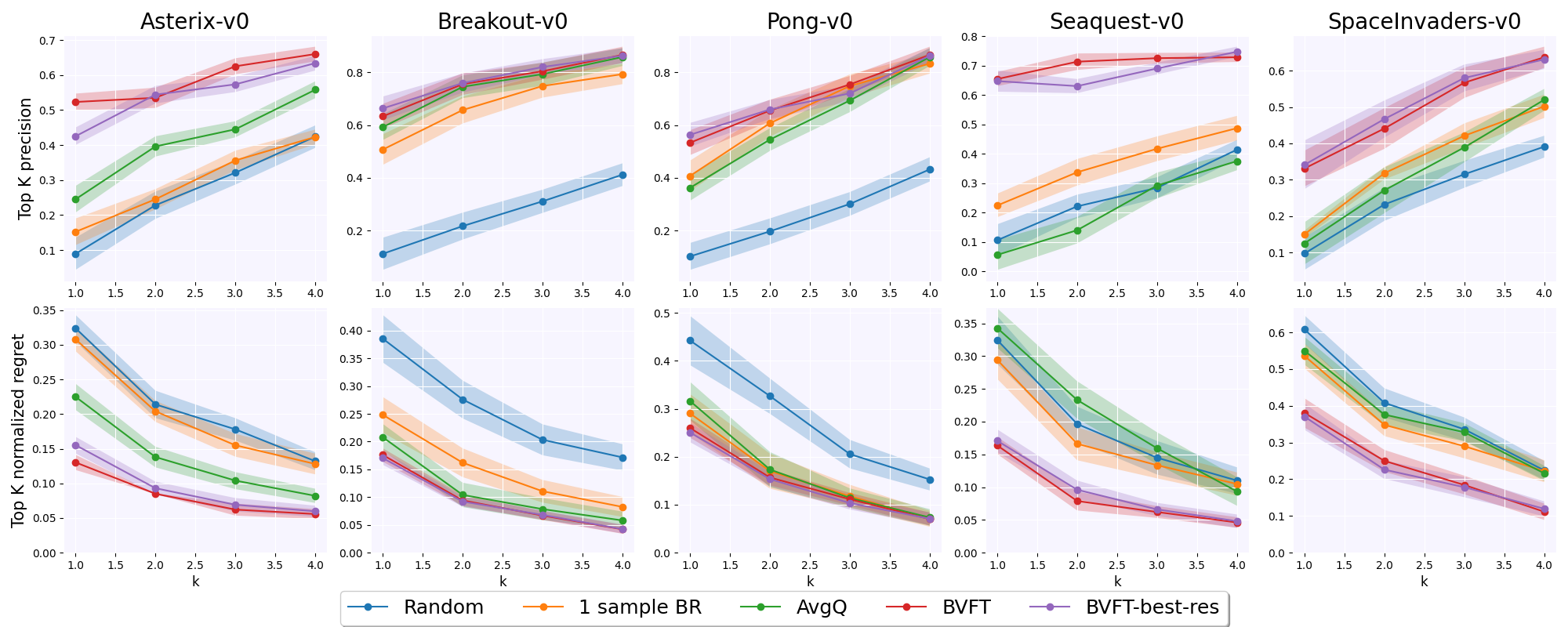}}
  \caption{Policy selection in Atari with candidate policies learned by CQL. Results are qualitatively similar to when training algorithms are BCQ (Figure~\ref{fig:atari-topk}). \label{fig:atari-cql}}
\end{figure}

\para{Mujoco with CQL as the training algorithm} Similarly, we also reproduce our results in Figure~\ref{fig:multi-ope}L with CQL as the training algorithm in Mujoco, and the results are qualitatively similar. See Figure~\ref{fig:mujoco-cql}. 
\begin{figure}
  \centering
\centerline{\includegraphics[width=0.6\columnwidth]{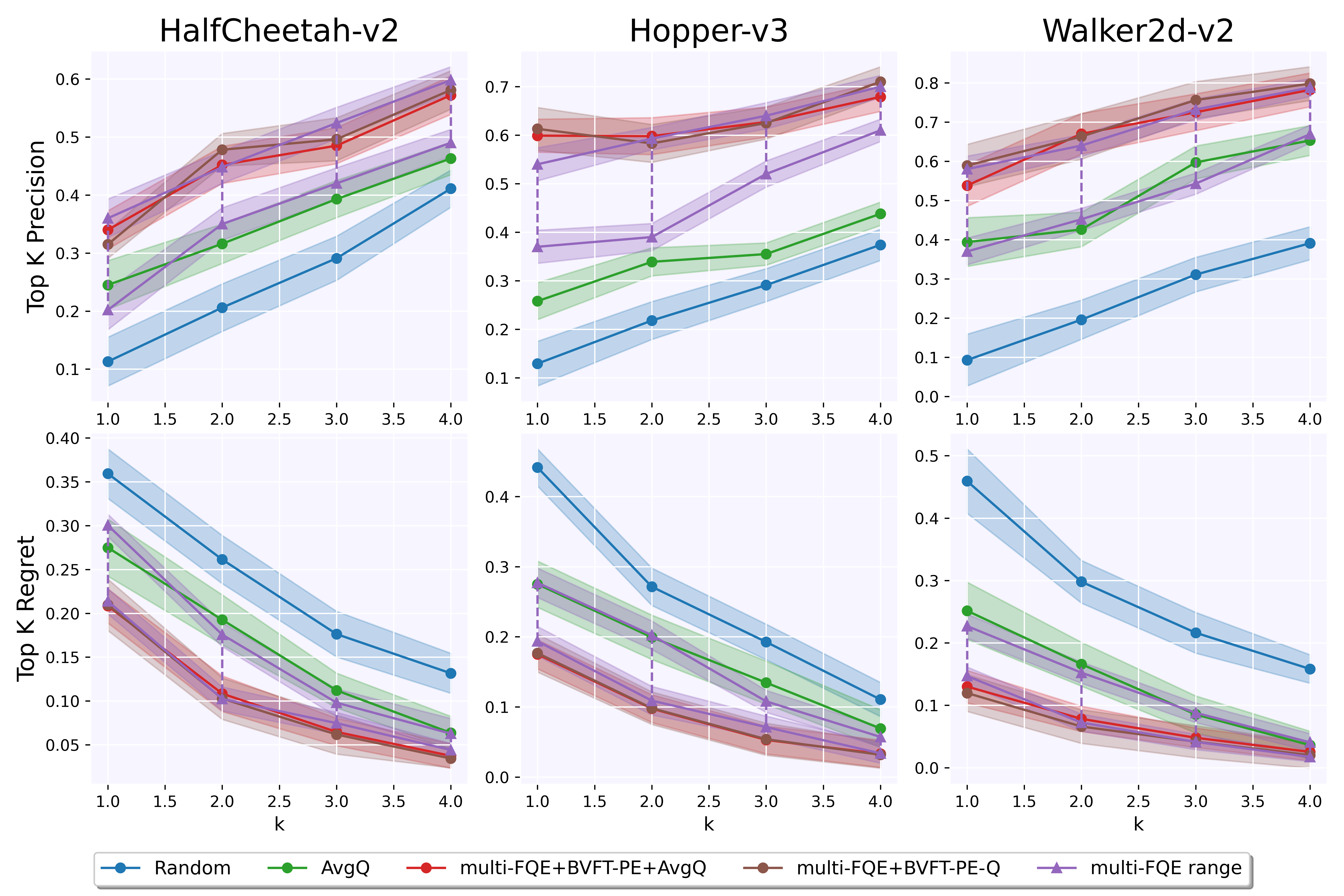}}
  \caption{Policy selection in Mujoco using \bvftpe and \bvftpeq with candidate policies learned by CQL. Results are qualitatively similar to Figure~\ref{fig:multi-ope}L). \label{fig:mujoco-cql}}
\end{figure}

\begin{figure}
  \centering
\centerline{\includegraphics[width=\columnwidth]{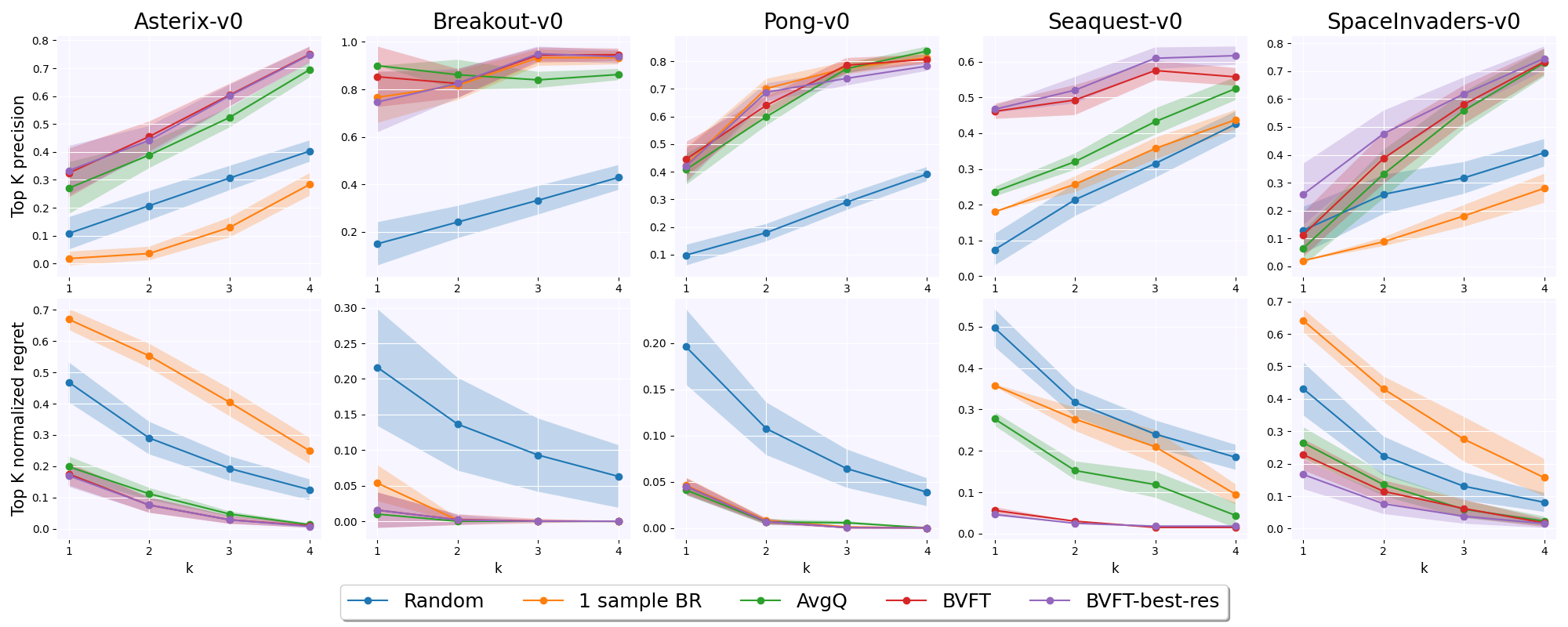}}
  \caption{Policy selection in Atari with candidate policies learned by (online) DQN. \label{fig:atari-dqn}}
\end{figure}

\para{Robustness w.r.t.~data exploratoriness} We test the sensitivity of different hyperparameter-free methods to data exploratoriness in two Atari domains. For this specific experiment, we generate our own data (unlike the other Atari experiments which use RLUnplugged datasets) by training an expert policy and mixing it with different probabilities of taking actions randomly. Figure~\ref{fig:explore} shows the robustness of \bvft across different amount of randomness in the data-generating policy. Similar robustness can be observed in Taxi (Figure~\ref{taxi-topk}, middle vs.~right). 

\para{Sample efficiency} Figure~\ref{fig:atari-top2-v-datasize} plots the performance of different methods against sample size for policy selection in Atari. Note that \bvft is competitive in the small-sample regime, and is the only method that improves when more data becomes available. 

\begin{figure}[hb]
\centering
  \centering
  \includegraphics[width=0.5\textwidth]{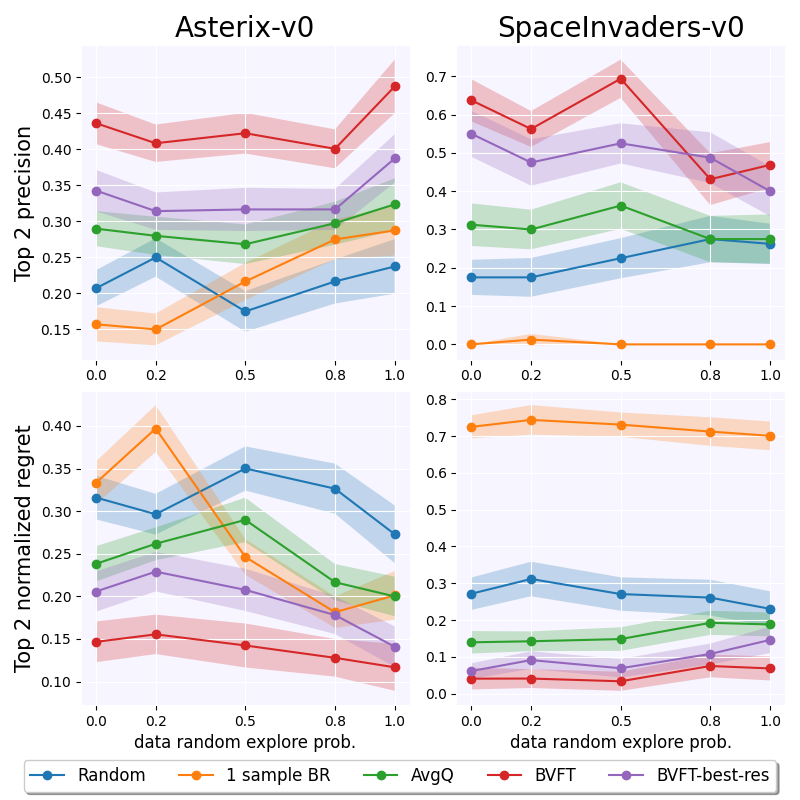}
  \label{fig:test2}
\caption{Sensitivity of policy-selection methods against data exploratoriness in 2 Atari games. X-axis is the probability of choosing a random action instead of the expert policy when generating the offline dataset for policy selection. Notice BVFT's performance is not sensitive to the level of stochasticity in the data-generation policy.\label{fig:explore} }
\end{figure}

\begin{figure}[ht]
\vskip 0.2in
\begin{center}
\centerline{\includegraphics[width=0.6\columnwidth]{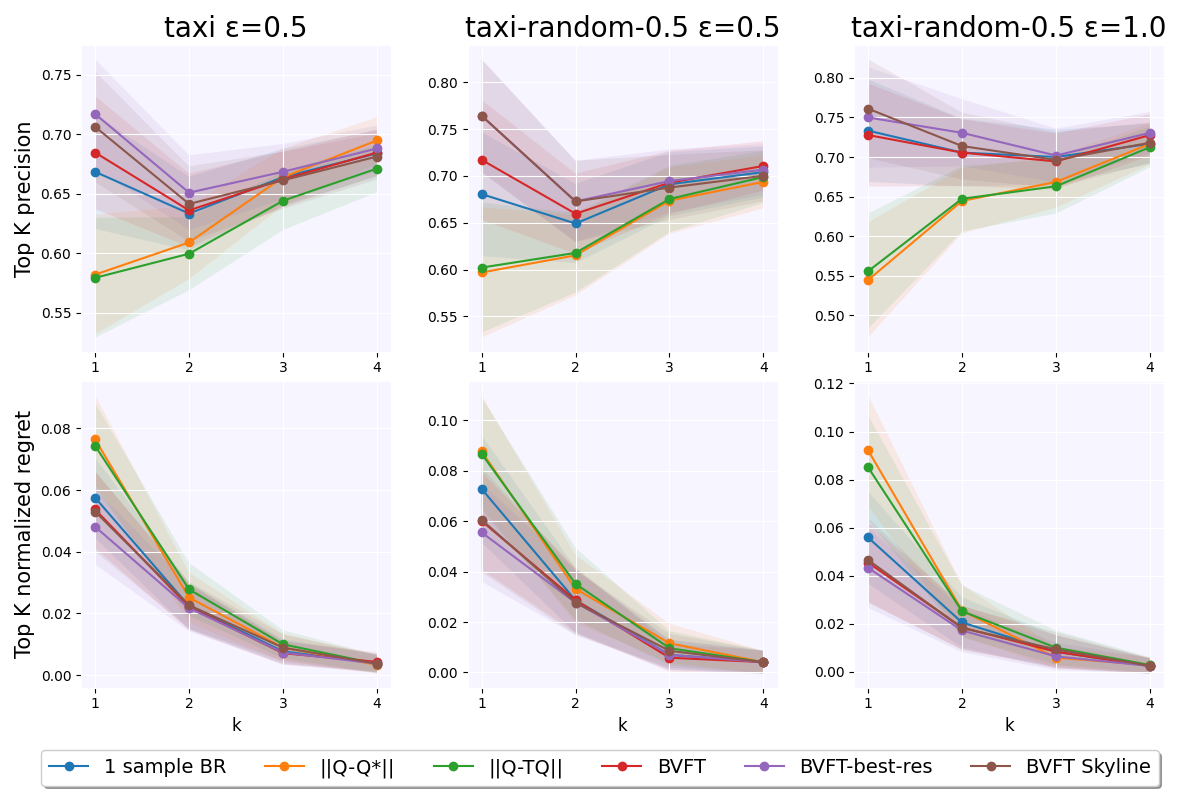}}
\caption{Policy selection in Taxi (left) and its modified version with additional stochasticity in transition dynamics (middle and right). $\epsilon$ indicates the stocahsticity in the data-generating policy.  ``\bvft Skyline'' is similar to \bvft, except that $\Gcal_{i,j}$ is produced by discretizing $\Scal\times\Acal$ according to $Q^\star$ itself instead of the compared functions. $\|Q-Q^\star\|$ and $\|Q-\Tcal Q\|$ are skylines that use the knowledge of $Q^\star$ and $\Tcal$, which are usually not available. 
\bvft competes favorably and sometimes outperforms these skylines. Baselines ``Random'' and ``AvgQ'' are not included as their performance is much worse and off the chart. Middle and right figures also show the robustness of the result w.r.t.~the data-generation policy.\label{taxi-topk} }
\end{center}
\vskip -0.2in
\end{figure}

\begin{figure}[ht]
\vskip 0.2in
\begin{center}
\centerline{\includegraphics[width=\columnwidth]{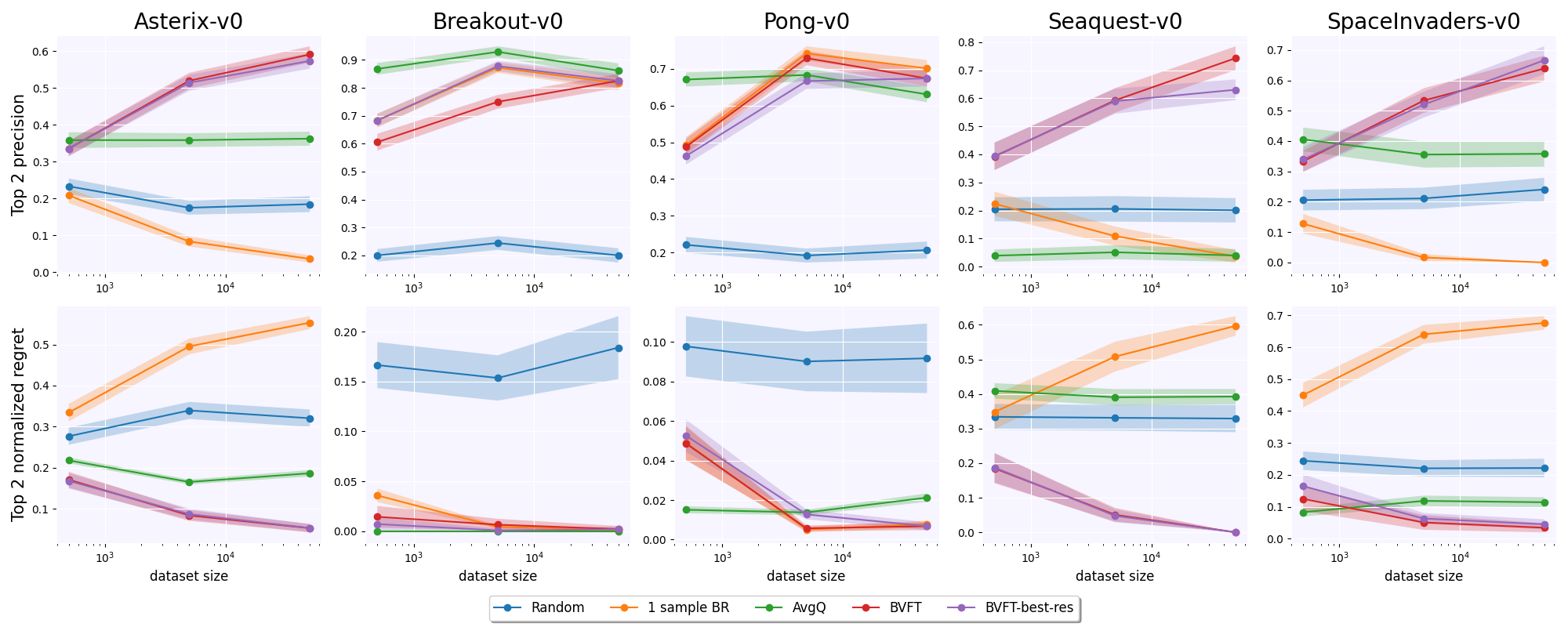}}
\caption{Top-2 precision and regret vs data size in 5 Atari environments. \bvft's performance improves as the data size increases while other baselines stays relatively constant.}
\label{fig:atari-top2-v-datasize}
\end{center}
\vskip -0.2in
\end{figure}


\end{document}